\documentclass{article} 

\usepackage[table,xcdraw]{xcolor}
\usepackage{hyperref}
\usepackage{url}            
\usepackage{booktabs}       
\usepackage{amsfonts}       
\usepackage{nicefrac}       
\usepackage{microtype}      
\usepackage{bm}
\usepackage{comment}
\usepackage{graphicx}
\usepackage{subfig}
\usepackage{enumitem}
\usepackage{footnote}

\makesavenoteenv{tabular}
\usepackage{mathtools}
\usepackage{amssymb}
\usepackage{latexsym}
\usepackage{dsfont}
\usepackage{mathrsfs}
\usepackage{amssymb}
\usepackage{amsfonts}
\usepackage{bm}
\usepackage{xspace}
\usepackage{amsthm}
\newcommand{\R}{\mathbb{R}} 

\newcommand{\diag}{\mathop{\mathrm{diag}}}

\newcommand{\hhat}{\hat{h}}

\newcommand{\pbf}{\mathbf{p}}

\newcommand{\wbf}{\mathbf{w}}

\newcommand{\ybf}{\mathbf{y}}
\newcommand{\zbf}{\mathbf{z}}

\newcommand{\Ahat}{\hat{A}}

\newcommand{\Dhat}{\hat{D}}

\newcommand{\Cbar}{\bar{C}}

\newcommand{\Lcal}{\mathcal{L}}

\newcommand{\Ebb}{\mathbb{E}}

\newcommand{\Rbb}{\mathbb{R}}

\ifx\BlackBox\undefined
\newcommand{\BlackBox}{\rule{1.5ex}{1.5ex}}  
\fi

\ifx\QED\undefined
\def\QED{~\rule[-1pt]{5pt}{5pt}\par\medskip}
\fi

\ifx\proof\undefined
\newenvironment{proof}{\par\noindent{\em Proof:\ }}{\hfill\BlackBox\\[.0mm]}
\fi

\ifx\theorem\undefined
\newtheorem{theorem}{Theorem}[section]
\fi

\ifx\example\undefined
\newtheorem{example}{Example}[section]
\fi

\ifx\property\undefined
\newtheorem{property}{Property}
\fi

\ifx\lemma\undefined
\newtheorem{lemma}{Lemma}[section]
\fi

\ifx\proposition\undefined
\newtheorem{proposition}{Proposition}[section]
\fi

\ifx\remark\undefined
\newtheorem{remark}{Remark}
\fi

\ifx\corollary\undefined
\newtheorem{corollary}{Corollary}[section]
\fi

\ifx\definition\undefined
\newtheorem{definition}{Definition}[section]
\fi

\ifx\conjecture\undefined
\newtheorem{conjecture}{Conjecture}[section]
\fi

\ifx\axiom\undefined
\newtheorem{axiom}[theorem]{Axiom}
\fi

\ifx\claim\undefined
\newtheorem{claim}[theorem]{Claim}
\fi

\ifx\assumption\undefined
\newtheorem{assumption}{Assumption}
\fi

\ifx\problem\undefined
\newtheorem{problem}{Problem}[section]
\fi

\ifx\fact\undefined
\newtheorem{fact}{Fact}
\fi

\newcommand{\rbr}[1]{\left(#1\right)}
\newcommand{\sbr}[1]{\left[#1\right]}
\newcommand{\cbr}[1]{\left\{#1\right\}}

\newcommand{\norm}[1]{\left\|#1\right\|}

\newcommand{\one}{\mathbf{1}}  
\newcommand{\zero}{\mathbf{0}} 

\newcommand{\relu}{{\mathrm{ReLU}}}

\usepackage[accepted]{icml2021}

\icmltitlerunning{GraphNorm: A Principled Approach to Accelerating Graph Neural Network Training}

\begin{document}

\twocolumn[
\icmltitle{GraphNorm: A Principled Approach to Accelerating \\Graph Neural Network Training}
\icmlsetsymbol{equal}{*}

\begin{icmlauthorlist}
\icmlauthor{Tianle Cai}{equal,princeton,haihua}
\icmlauthor{Shengjie Luo}{equal,moe,pku,pazhou}
\icmlauthor{Keyulu Xu}{mit}
\icmlauthor{Di He}{ms}
\icmlauthor{Tie-Yan Liu}{ms}
\icmlauthor{Liwei Wang}{moe,pku}

\end{icmlauthorlist}
\icmlaffiliation{princeton}{Department of Electrical and Computer Engineering, Princeton University}
\icmlaffiliation{haihua}{Zhongguancun Haihua Institute for Frontier Information Technology}
\icmlaffiliation{moe}{Key Laboratory of Machine Perception, MOE, School of EECS, Peking University}
\icmlaffiliation{pazhou}{Pazhou Lab}
\icmlaffiliation{mit}{CSAIL, Massachusetts Institute of Technology (MIT)}
\icmlaffiliation{ms}{Microsoft Research}
\icmlaffiliation{pku}{Center for Data Science, Peking University}

\icmlcorrespondingauthor{Liwei Wang}{wanglw@pku.edu.cn}
\icmlcorrespondingauthor{Di He}{dihe@microsoft.com}
\icmlkeywords{Graph Neural Network, Normalization}

\vskip 0.3in
]
\printAffiliationsAndNotice{\icmlEqualContribution}
\begin{abstract}
Normalization is known to help the optimization of deep neural networks. Curiously, different architectures require specialized normalization methods. In this paper, we study what normalization is effective for Graph Neural Networks (GNNs). First, we adapt and evaluate the existing methods from other domains to GNNs. Faster convergence is achieved with InstanceNorm compared to BatchNorm and LayerNorm. We provide an explanation by showing that InstanceNorm serves as a preconditioner for GNNs, but such preconditioning effect is weaker with BatchNorm due to the heavy batch noise in graph datasets. Second, we show that the shift operation in InstanceNorm results in an expressiveness degradation of GNNs for highly regular graphs. We address this issue by proposing GraphNorm with a learnable shift. Empirically, GNNs with GraphNorm converge faster compared to GNNs using other normalization. GraphNorm also improves the generalization of GNNs, achieving better performance on graph classification benchmarks.
\end{abstract}

\section{Introduction}
\label{sec:intro}

\begin{figure*}[ht]
    \centering
        \includegraphics[width=0.8\textwidth]{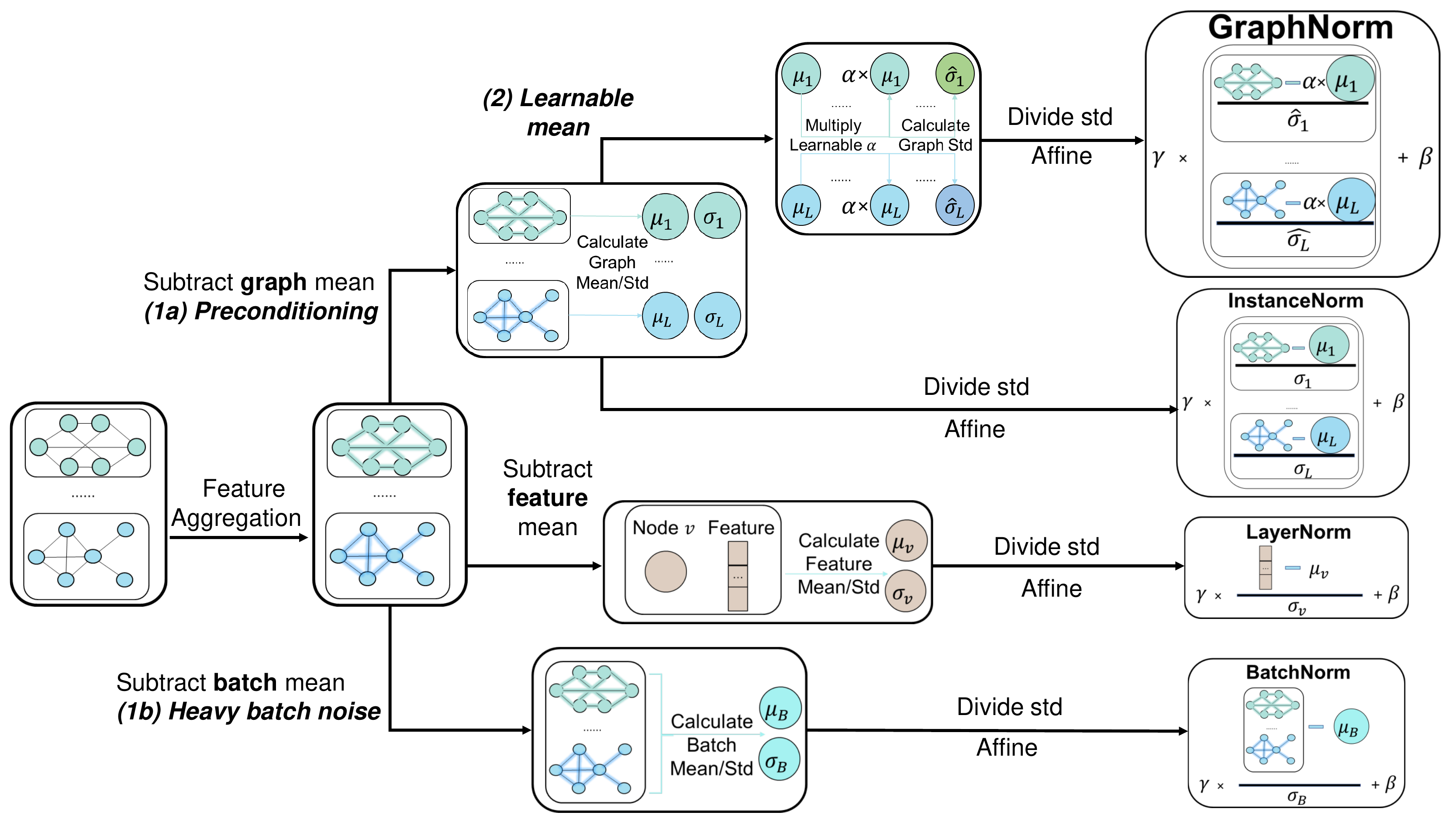}
    \caption{\textbf{Overview.} 
We evaluate and understand BatchNorm, LayerNorm, and InstanceNorm, when adapted to GNNs. InstanceNorm trains faster than LayerNorm and BatchNorm on most datasets (Section~\ref{sec:adapt_to_gnns}), as it serves as a preconditioner of the aggregation of GNNs (1a, Section~\ref{sec:precondition}). The preconditioning effect is weaker for BatchNorm  due to heavy batch noise in graphs (1b, Section~\ref{sec:batchnorm_noise}). We propose GraphNorm with a learnable shift to address the  limitation of InstanceNorm. GraphNorm outperforms other normalization methods for both training speed (Figure~\ref{fig:gin-dataset-training-curve}) and generalization (Table~\ref{tab:test-results},~\ref{tab:test-molhiv}).
}
\label{fig:illustration}
\end{figure*}
Recently, there has been a surge of interest in Graph Neural Networks (GNNs) for learning with graphs~\citep{gori2005new, scarselli2008graph,hamilton2017inductive,kipf2016semi,velivckovic2017graph, xu2018representation,ying2021transformers}. GNNs learn node and graph representations by recursively aggregating and updating the node representations from neighbor representations~\citep{gilmer2017neural}. Empirically, GNNs have succeeded in a variety of tasks such as computational chemistry~\citep{stokes2020deep}, recommendation systems~\citep{ying2018graph}, and visual question answering~\citep{santoro2017simple}. Theoretically, existing works have studied GNNs through the lens of expressive power~\citep{keriven2019universal, xu2018how, sato2019approximation, Loukas2020,ying2021transformers}, generalization~\citep{scarselli2018vapnik, du2019graph, Xu2020What}, and extrapolation~\citep{xu2020neural}. However, the optimization of GNNs is less well understood, and in practice, the training of GNNs is often unstable and the convergence is slow~\citep{xu2018how}. 

In this paper, we study how to improve the training of GNNs via normalization. Normalization methods shift and scale the hidden representations and are shown to help the optimization for deep neural networks~\citep{ioffe2015batch,ulyanov2016instance, ba2016layer,salimans2016weight,xiong2020layer,salimans2016improved,miyato2018spectral,wu2018group, santurkar2018does}. Curiously, no single normalization helps in every domain, and different architectures require specialized methods. For example, Batch normalization (BatchNorm) is a standard component in computer vision~\citep{ioffe2015batch}; Layer normalization (LayerNorm) is popular in natural language processing~\citep{ba2016layer,xiong2020layer}; Instance normalization (InstanceNorm) has been found effective for style transfer tasks~\citep{ulyanov2016instance} . This motivates the question: \textit{What normalization methods are effective for GNNs?}

We take an initial step towards answering the question above. First, we adapt the existing  methods from other domains, including BatchNorm, LayerNorm, and InstanceNorm, to GNNs and evaluate their performance with extensive experiments on graph classification tasks. We observe that our adaptation of InstanceNorm to GNNs, which for each \emph{individual graph} normalizes its node hidden representations, obtains much faster convergence compared to BatchNorm and LayerNorm. We provide an explanation for the success of InstanceNorm by showing that the shift operation in InstanceNorm serves as a preconditioner of the graph aggregation operation. Empirically, such preconditioning makes the optimization curvature smoother and makes the training more efficient. We also explain why the widely used BatchNorm does not bring the same level of acceleration. The variance of the batch-level statistics on graph datasets is much larger if we apply the normalization across graphs in a batch instead of across individual graphs. The noisy statistics during training may lead to unstable optimization.

Second, we show that the adaptation of InstanceNorm to GNNs, while being helpful in general, has limitations. The shift operation in InstanceNorm, which subtracts the mean statistics from node hidden representations, may lead to an expressiveness degradation for GNNs. Specifically, for highly regular graphs, the mean statistics contain graph structural information, and thus removing them could hurt the performance. Based on our analysis, we propose \textit{GraphNorm} to address the issue of InstanceNorm with a learnable shift (Step 2 in Figure~\ref{fig:illustration}). The learnable shift could learn to control the ideal amount of information to preserve for mean statistics. Together, GraphNorm normalizes the hidden representations across nodes in each individual graph with a learnable shift to avoid the expressiveness degradation while inheriting the acceleration effect of the shift operation.

We validate the effectiveness of GraphNorm on eight popular graph classification benchmarks. Empirical results confirm that GraphNorm consistently improves the speed of converge and stability of training for GNNs compared to those with BatchNorm, InstanceNorm, LayerNorm, and those without normalization. Furthermore, GraphNorm helps GNNs achieve better generalization performance on most benchmarks.

\subsection{Related Work}
\label{sec:related}
Closely related to our work, InstanceNorm~\citep{ulyanov2016instance} is originally proposed for real-time image generation. Variants of InstanceNorm are also studied in permutation equivalent data processing~\citep{yi2018learning,sun2020acne}. We instead adapt InstanceNorm to GNNs and find it helpful for the training of GNNs. Our proposed GraphNorm builds on and improves InstanceNorm by addressing its expressiveness degradation with a learnable shift. 

Few works have studied normalization in the GNN literature. \citet{xu2018how} adapts BatchNorm to GIN as a plug-in component. A preliminary version of \citet{dwivedi2020benchmarking} normalizes the node features with respect to the graph size. Our GraphNorm is size-agnostic and significantly differs from the graph size normalization. More discussions on other normalization methods are in Appendix~\ref{appsec:related}.

The reason behind the effectiveness of normalization has been intensively studied. While scale and shift are the main components of normalization, most existing works focus on the scale operation and the ``scale-invariant'' property: With a normalization layer after a linear (or convolutional) layer, the output values remain the same as the weights are scaled.  Hence, normalization decouples the optimization of direction and length of the parameters \citep{kohler2019exponential}, implicitly tunes the learning rate \citep{ioffe2015batch,hoffer2018norm,arora2018theoretical,li2019exponential}, and smooths the optimization landscape \citep{santurkar2018does}. Our work offers a different view by instead showing specific \textit{shift} operation has the preconditioning effect and can accelerate the training of GNNs.

\section{Preliminaries}
\label{sec:background}
We begin by introducing our notations and the basics of GNNs. Let $G = \left(V, E \right)$ denote a graph where $V = \{v_1, v_2, \cdots, v_n\}$, $n$ is the number of nodes. Let the feature vector of node $v_i$ be $X_i$. We denote the adjacency matrix of a graph as $A\in\Rbb^{n\times n}$ with $A_{ij}=1$ if $(v_i,v_j)\in E$ and $0$ otherwise. The degree matrix associated with $A$ is defined as $D=\diag\rbr{d_1, d_2, \dots, d_n}$ where $d_i=\sum_{j=1}^{n} A_{ij}$.

{\bf Graph Neural Networks.} GNNs use the graph structure and node features to learn the representations of nodes and graphs. Modern GNNs follow a neighborhood aggregation strategy \citep{sukhbaatar2016learning,kipf2016semi,hamilton2017inductive,velivckovic2017graph,monti2017geometric,ying2021transformers}, where the representation of a node is iteratively updated by aggregating the representation of its neighbors. To be concrete, we denote $h^{(k)}_i$ as the representation of $v_i$ at the $k$-th layer and define $h_i^{(0)} = X_i$. We use AGGREGATE to denote the aggregation function in the $k$-th layer:
\begin{align}
    \label{eq:combine}
    h_i^{(k)}   &= \text{AGGREGATE}^{(k)} \big( h_i^{(k-1)}, \big\lbrace h_j^{(k-1)}  : v_j \in \mathcal{N}(v_i) \big\rbrace \big), 
\end{align}
where $\mathcal{N}(v_i)$ is the set of nodes adjacent to $v_i$. Different GNNs can be obtained by choosing different AGGREGATE functions.  Graph Convolutional Networks (GCN)~\citep{kipf2016semi} can be defined in matrix form as: 
\begin{align}
\label{equ:gcn-agg-matrix}
    H^{(k)} = \relu\rbr{W^{(k)} H^{(k-1)} Q_{\text{GCN}}},
\end{align}
where ReLU stands for rectified linear unit, $H^{(k)} = \sbr{h_1^{(k)}, h_2^{(k)}, \cdots, h_n^{(k)}}\in\Rbb^{d^{(k)}\times n}$ is the feature matrix at the $k$-th layer where $d^{(k)}$ denotes the feature dimension, and $W^{(k)}$ is the parameter matrix in layer $k$. $Q_{\text{GCN}}=\Dhat^{-\frac{1}{2}}\Ahat\Dhat^{-\frac{1}{2}}$, where $\Ahat = A + I_n$ and $\Dhat$ is the degree matrix of $\Ahat$. $I_n$ is the identity matrix. 


Graph Isomorphism Network (GIN) \citep{xu2018how} is defined in matrix form as 
\begin{align}
    \label{eq:gin-agg-matrix}
    H^{(k)} = {\rm MLP}^{(k)}\rbr{W^{(k)} H^{(k-1)}Q_{\mathrm{GIN}}},
\end{align}
where MLP stands for multilayer perceptron, $\xi^{(k)}$ is a learnable parameter and $Q_{\mathrm{GIN}}=A + I_n + \xi^{(k)}I_n$.

For a $K$-layer GNN, the outputs of the final layer, i.e., $h_i^{(K)}$,$i=1,\cdots,n$, will be used for prediction. For graph classification tasks, we can apply a READOUT function, e.g., summation, to aggregate node features $h_i^{(K)}$ to obtain the entire graph's representation $h_{G} = {\rm READOUT} \big(\big\lbrace h_i^{(K)} \ \big\vert \ v_i \in V \big\rbrace \big)$. A classifier can be applied upon $h_{G}$ to predict the labels.

{\bf Normalization.} Generally, given a set of values $ \cbr{x_1, x_2, \cdots, x_m}$, a normalization operation first shifts each $x_i$ by the mean $\mu$, and then scales them down by standard deviation $\sigma$: $x_i\rightarrow\gamma \frac{x_i - \mu}{\sigma}+\beta$, where $\gamma$ and $\beta$ are learnable parameters, $\mu = \frac{1}{m}\sum_{i=1}^m x_i$ and $\sigma^2 = \frac{1}{m}\sum_{i=1}^m\rbr{x_i - \mu}^2$. The major difference among different existing normalization methods is which set of feature values the normalization is applied to. 
For example, in computer vision, BatchNorm normalizes the feature values in the same channel across different samples in a batch. In NLP,  LayerNorm normalizes the feature values at each position in a sequence separately.



\section{Evaluating and Understanding Normalization for GNNs}
\label{sec:normalization-gnn}
In this section, we first adapt and evaluate existing normalization methods to GNNs. Then we give an explanation of the effectiveness of the variant of InstanceNorm, and show why the widely used BatchNorm fails to have such effectiveness.  The  understanding inspires us to develop better normalization methods, e.g., GraphNorm.

\subsection{Adapting and Evaluating Normalization for GNNs}
\label{sec:adapt_to_gnns}
\begin{figure*}[t]
    \centering
        \includegraphics[width=\textwidth]{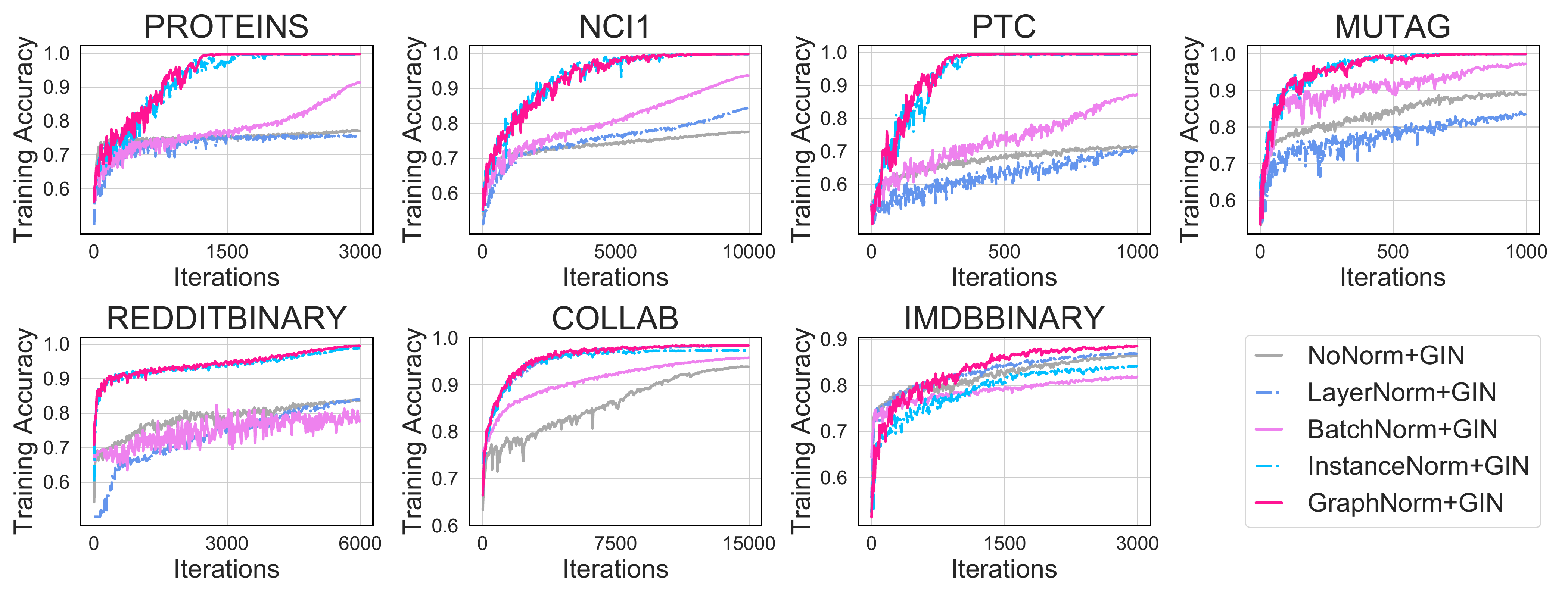}
        
    \caption{\textbf{Training performance} of GIN with different normalization methods and GIN without normalization in graph classification tasks. The convergence speed of our adaptation of InstanceNorm dominates BatchNorm and LayerNorm in most tasks. GraphNorm further improves the training over InstanceNorm especially on tasks with highly regular graphs, e.g., IMDB-BINARY (See Figure~\ref{fig:ablation-study-learn-alpha} for detailed illustration). Overall, GraphNorm converges faster than all other methods. 
    }
\label{fig:gin-dataset-training-curve}
\end{figure*}
To investigate what normalization methods are effective for GNNs, we first adapt three typical normalization methods, i.e., BatchNorm, LayerNorm, and InstanceNorm, developed in other domain to GNNs.  We apply the normalization after the linear transformation as in previous works~\citep{ioffe2015batch,xiong2020layer,xu2018how}. The general GNN structure equipped with a normalization layer can be represented as:
\begin{align}
    \label{eq:structure_with_norm}
    H^{(k)} = F^{(k)}\rbr{\mathrm{Norm}\rbr{W^{(k)}H^{(k-1)}Q}},
\end{align}
where $F^{(k)}$ is a function that applies to each node separately, $Q$ is an $n\times n$ matrix representing the neighbor aggregation, and $W^{(k)}$ is the weight/parameter matrix in layer $k$. We can instantiate Eq.~\eqref{eq:structure_with_norm} as GCN and GIN, by setting proper  $F^{(k)}$ and matrix $Q$. For example, if we  set $F^{(k)}$ to be $\relu$ and set $Q$ to be $
Q_{\text{GCN}}$ (Eq.~\eqref{equ:gcn-agg-matrix}), then Eq.~\eqref{eq:structure_with_norm} becomes GCN  with normalization; Similarly, by setting $F^{(k)}$ to be $\mathrm{MLP}^{(k)}$ and $Q$ to be $Q_{\mathrm{GIN}}$ (Eq.~\eqref{eq:gin-agg-matrix}), we recover GIN with normalization.

We then describe the concrete operations of the adaptations of the normalization methods. Consider a batch of graphs $\cbr{G_1, \cdots, G_b}$ where $b$ is the batch size. Let $n_g$ be the number of nodes in graph $G_g$. We generally denote $\hhat_{i,j,g}$ as the inputs to the normalization module, e.g., the $j$-th feature value of node $v_i$ of graph $G_g$, $i=1, \cdots, n_g, j=1, \cdots, d, g=1, \cdots, b$. The adaptations take the general form:
\begin{align}\label{eq:general_form_norm}
    \mathrm{Norm}\rbr{\hhat_{i,j,g}} = \gamma \cdot \frac{\hhat_{i,j,g} - \mu}{\sigma} + \beta,
\end{align}
where the scopes of mean $\mu$, standard deviation $\sigma$, and affine parameters $\gamma, \beta$ differ for different normalization methods. For BatchNorm, normalization and the computation of $\mu$ and $\sigma$ are applied to all values in the same feature dimension across the nodes of \emph{all graphs in the batch} as in \citet{xu2018how}, i.e., over dimensions $g,i$ of $\hhat_{i,j,g}$. 
 To adapt LayerNorm to GNNs, we view each node as a basic component, resembling words in a sentence, and apply normalization to all feature values across different dimensions of each node, i.e., over dimension $j$ of $\hhat_{i,j,g}$.
For InstanceNorm, we regard each graph as an instance. The normalization is then applied to the feature values across all nodes for each \emph{individual graph}, i.e., over dimension $i$ of $\hhat_{i,j,g}$.

In Figure~\ref{fig:gin-dataset-training-curve} we show training curves of different normalization methods in graph classification tasks. We find that LayerNorm hardly improves the training process in most tasks, while  our adaptation of InstanceNorm can largely boost the training speed compared to other normalization methods. The test performances have similar trends. We summarize the final test accuracies in Table~\ref{tab:test-results}. In the following subsections, we provide an explanation for the success of InstanceNorm and its benefits compared to BatchNorm, which is currently adapted in many GNNs.

\subsection{Shift in InstanceNorm as a Preconditioner}
\label{sec:precondition}
As mentioned in Section~\ref{sec:related}, the scale-invariant property of the normalization has been investigated and considered as one of the ingredients that make the optimization efficient. In our analysis of normalizations for GNNs, we instead take a closer look at the \emph{shift} operation in the normalization. Compared to the image and sequential data, the graph is explicitly structured, and the neural networks exploit the structural information directly in the aggregation of the neighbors, see Eq.~\eqref{eq:combine}. Such uniqueness of GNNs makes it possible to study how the shift operation interplays with the graph data in detail.

We show that the shift operation in our adaptation of InstanceNorm serves as a preconditioner of the aggregation in GNNs and hypothesize this preconditioning effect can boost the training of GNNs. Though the current theory of deep learning has not been able to prove and compare the convergence rate in the real settings, we calculate the convergence rate of GNNs on a simple but fully characterizable setting to give insights on the benefit of the shift operation.

fWe first formulate our adaptation of InstanceNorm in the matrix form. Mathematically, for a graph of $n$ nodes, denote $N = I_n - \frac{1}{n}\one\one^\top$. $N$ is the matrix form of the shift operation, i.e., for any vector $\zbf=\sbr{z_1,z_2,\cdots,z_n}^\top\in\Rbb^n$, $\zbf^\top N = \zbf^\top - \rbr{\frac{1}{n}\sum_{i=1}^n z_i}\one^\top$. Then the normalization together with the aggregation can be represented as\footnote{Standard normalization has an additional affine operation after shifting and scaling. Here we omit it in Eq.~\ref{eq:structure_with_norm_exp} for better 
demonstration. Adding this operation will not affect the theoretical analysis.}
\begin{align} \label{eq:structure_with_norm_exp}
    \mathrm{Norm}\rbr{W^{(k)}H^{(k-1)}Q} = S\rbr{W^{(k)}H^{(k-1)}Q}N,
\end{align}
where $S = \diag\rbr{\frac{1}{\sigma_1}, \frac{1}{\sigma_2}, \cdots, \frac{1}{\sigma_{d^{(k)}}}}$ is the scaling, and $Q$ is the GNN aggregation matrix. Each $\sigma_i$ is the standard deviation of the values of the $i$-th features among the nodes in the graph we consider. We can see that, in the matrix form, shifting feature values on a single graph is equivalent to multiplying $N$ as in Eq.~\eqref{eq:structure_with_norm_exp}. Therefore, we further check how this operation affects optimization. In particular, we examine the singular value distribution of $Q N$. The following theorem shows that $Q N$ has a smoother singular value distribution than $Q$, i.e., $N$ serves as a preconditioner of $Q$. 
\begin{theorem}[Shift Serves as a Preconditioner of $Q$]
    \label{thm:precondition}
    Let $Q, N$ be defined as in Eq.~\eqref{eq:structure_with_norm_exp},  $0\le\lambda_1\le\cdots\le\lambda_n$ be the singular values of $Q$. We have $\mu_n=0$ is one of the singular values of $QN$, and let other singular values of $QN$ be $0\le\mu_1\le\mu_2\le\cdots\le\mu_{n-1}$. Then we have
    \begin{align} \label{eq:mean_sub_eigen}
        \lambda_1\le\mu_1\le\lambda_2\le\cdots\le\lambda_{n-1}\le\mu_{n-1}\le\lambda_n,
    \end{align}
    where $\lambda_i = \mu_i$ or $\lambda_i = \mu_{i-1}$ only if there exists one of the right singular vectors $\alpha_i$ of $Q$ associated with $\lambda_i$ satisfying $\one^\top\alpha_i = 0$.
\end{theorem}
The proof can be found in Appendix~\ref{appsec:proof_precondition}.

We hypothesize that precoditioning $Q$ can help the optimization. In the case of optimizing the weight matrix $W^{(k)}$, we can see from Eq.~\eqref{eq:structure_with_norm_exp} that after applying normalization, the term $Q$ in the gradient of $W^{(k)}$ will become $Q N$ which makes the optimization curvature of $W^{(k)}$ smoother, see Appendix~\ref{appsec:gradient_formular} for more discussions. Similar preconditioning effects are believed to improve the training of deep learning models~\citep{duchi2011adaptive, kingma2014adam}, and classic wisdom in optimization has also shown that preconditioning can accelerate the convergence of iterative methods \citep{axelsson1985survey, demmel1997applied}. Unfortunately, current theoretical toolbox only has a limited power on the optimization of deep learning models. Global convergence rates have only been proved for either simple models, e.g., linear models \citep{arora2018convergence}, or extremely overparameterized models \citep{du2018gradient,allen2019convergence,du2019gradient,cai2019gram,du2019graph,zou2020gradient}. To support our hypothesis that preconditioning may suggest better training, we investigate a simple but characterizable setting of training a linear GNN using gradient descent in Appendix~\ref{appsec:concrete_example}. In this setting, we prove that:
\begin{proposition}[Concrete Example Showing Shift can Accelerate Training (Informal)]
With high probability over randomness of data generation, the parameter $\wbf_t^{\mathrm{Shift}}$ of the model with shift at step $t$ converges to the optimal parameter $\wbf_*^{\mathrm{Shift}}$ linearly:
\begin{align*}
    \norm{\wbf_t^{\mathrm{Shift}} - \wbf_*^{\mathrm{Shift}}}_2 = O\rbr{\rho_1^t},
\end{align*}
where $\rho_1$ is the convergence rate.

Similarly, the parameter $\wbf_t^{\mathrm{Vanilla}}$ of the vanilla model converges linearly, but with a slower rate:
\begin{align*}
    \norm{\wbf_t^{\mathrm{Vanilla}} - \wbf_*^{\mathrm{Vanilla}}}_2 = O\rbr{\rho_2^t}\ \mathrm{and}\ \rho_1 < \rho_2,
\end{align*}
which indicates that the model with shift converges faster than the vanilla model.
\end{proposition}
The proof can be found in Appendix~\ref{appsec:concrete_example}.
\begin{figure}[t]
    \centering
        \includegraphics[width=0.48\textwidth]{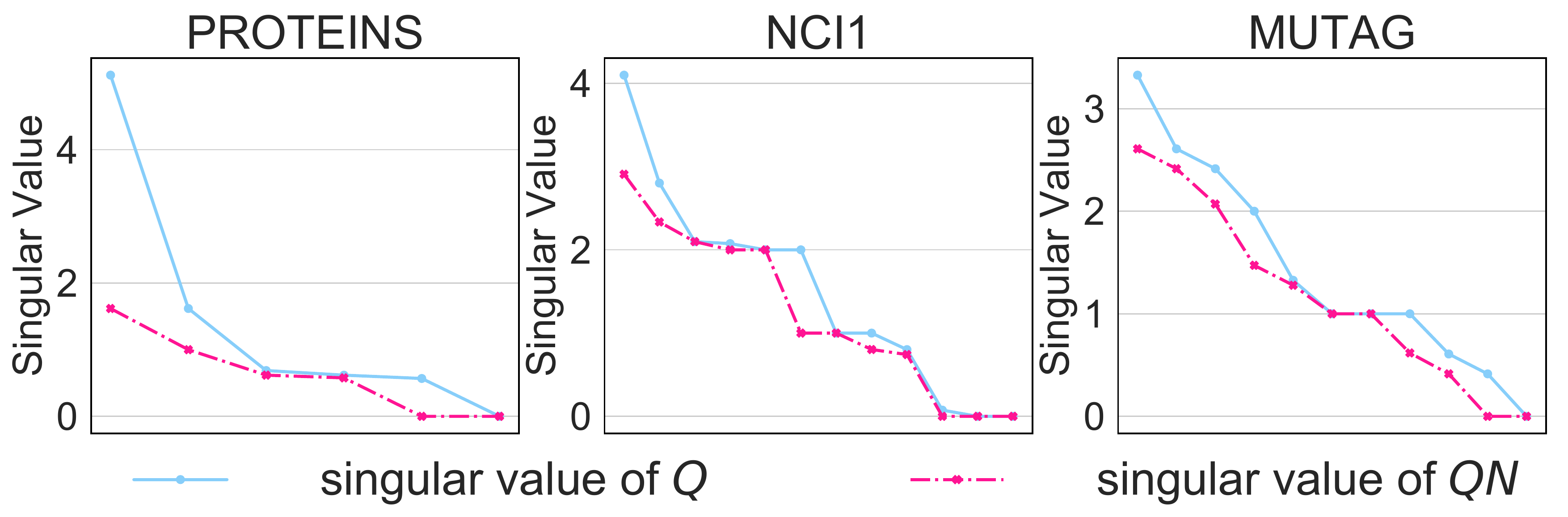}
    \caption{\textbf{Singular value distribution} of $Q$ and $Q N$ for sampled graphs in different datasets using GIN. More visualizations  can be found in Appendix \ref{appsec:vis-singular}}
\label{fig:singular_distribution}
\end{figure}
To check how much the matrix $N$ improves the distribution of the spectrum of matrix $Q$ in real practice, we sample graphs from different datasets for illustration, as showed in Figure \ref{fig:singular_distribution} (more visualizations for different types of graph can be found in Appendix \ref{appsec:vis-singular}). We can see that the singular value distribution of $Q N$ is much smoother, and the condition number is improved. Note that for a multi-layer GNN, the normalization will be applied in each layer. Therefore, the overall improvement of such preconditioning can be more significant.

\subsection{Heavy Batch Noise in Graphs Makes BatchNorm Less Effective}\label{sec:batchnorm_noise}
\begin{figure*}[t]
    \centering
        \includegraphics[width=\textwidth]{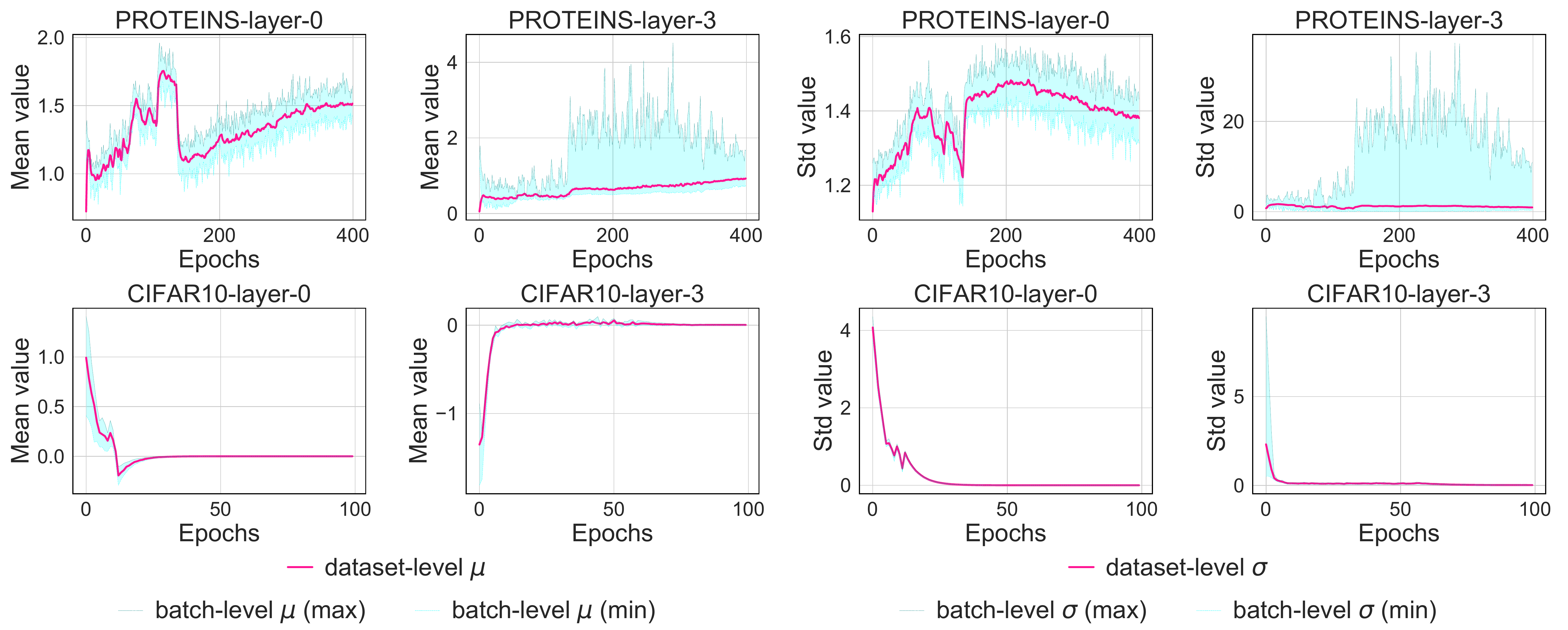}
        
    \caption{
    \textbf{Batch-level statistics are noisy for GNNs.} 
    We plot the batch-level/dataset-level mean/standard deviation of models trained on PROTEINS (graph classification) and CIFAR10 (image classification). We observe that the deviation of batch-level statistics from dataset-level statistics is rather large for the graph task, while being negligible in image task.
    }
\label{fig:noise-comparison}
\end{figure*}

The above analysis shows the adaptation of InstanceNorm has the effect of preconditioning the aggregation of GNNs. Then a natural question is whether a batch-level normalization for GNNs \citep{xu2018how} has similar advantages. We show that BatchNorm is less effective in GNNs due to heavy batch noise on graph data.

In BatchNorm, the mean $\mu_B$ and standard deviation $\sigma_B$ are calculated in a sampled batch during training, which can be viewed  as random variables by the randomness of sampling. During testing, the estimated dataset-level statistics (running mean $\mu_D$ and standard deviation $\sigma_D$) are used instead of the batch-level statistics \citep{ioffe2015batch}. To apply Theorem~\ref{thm:precondition} to BatchNorm for the preconditioning effect, one could potentially view all graphs in a dataset as  subgraphs in a \emph{super graph}. Hence, Theorem~\ref{thm:precondition} applies to BatchNorm if the batch-level statistics are well-concentrated around dataset-level statistics, i.e., $\mu_B\approx\mu_D$ and $\sigma_B\approx\sigma_D$. However, the concentration of batch-level statistics is heavily \textit{domain-specific}. While \citet{shen2020powernorm} find the variation of batch-level statistics in typical networks is small for computer vision,  the concentration of batch-level statistics  is still unknown for GNNs. 

We study how the batch-level statistics $\mu_B, \sigma_B$ deviate from the dataset-level statistics $\mu_D, \sigma_D$. For comparison, we train a 5-layer GIN with BatchNorm on the PROTEINS dataset and train a ResNet18 \citep{he2016deep} on the CIFAR10 dataset. We set batch size to 128. For each epoch, we record the batch-level max/min mean and standard deviation for the first  and the last BatchNorm layer on a randomly selected dimension across batches. In Figure \ref{fig:noise-comparison}, pink line denotes the dataset-level statistics, and green/blue line denotes the max/min value of the batch-level statistics. We observe that for image tasks, the maximal deviation of the batch-level statistics from the dataset-level statistics is negligible (Figure~\ref{fig:noise-comparison}) after a few epochs. In contrast, for the graph tasks, the variation of batch-level statistics stays  large during training. Intuitively, the graph structure can be quite diverse and the a single batch cannot well represent the entire dataset. Hence, the preconditioning property also may not hold for BatchNorm. In fact, the heavy batch noise  may bring instabilities to the training. More results may be found in Appendix \ref{appsec:vis-noise}.

\section{Graph Normalization}
\label{sec:graphnorm}

Although we provide evidence on the indispensability and advantages of our adaptation of InstanceNorm, simply normalizing the values in each feature dimension within a graph does not consistently lead to improvement. We show that in some situations, e.g., for regular graphs, the standard shift (e.g., shifting by subtracting the mean) may cause information loss on graph structures. 

We consider $r$-regular graphs, i.e., each node has a degree $r$. We first look into the case that there are no available node features, then $X_i$ is set to be the one-hot encoding of the node degree \citep{xu2018how}. In a $r$-regular graph, all nodes have the same encoding, and thus the columns of $H^{(0)}$ are the same. We study the output of the standard shift operation in the first layer, i.e., $k=1$ in Eq.~\eqref{eq:structure_with_norm_exp}. From the following proposition, we can see that when the standard shift operation is applied to GIN for a $r$-regular graph described above, the information of degree is lost:
\begin{proposition}\label{prop:regular_graph}
For a $r$-regular graph with one-hot encodings as its features described above, we have for GIN, $\mathrm{Norm}\rbr{W^{(1)}H^{(0)}Q_{\mathrm{GIN}}} = S\rbr{W^{(1)}H^{(0)}Q_{\mathrm{GIN}}}N = 0$,
i.e., the output of normalization layer is a zero matrix without any information of the graph structure.
\end{proposition}
Such information loss not only happens when there are no node features. For complete graphs, we can further show that even each node has different features, the graph structural information, i.e., adjacency matrix $A$, will always be ignored after the standard shift operation in GIN:
\begin{proposition}\label{prop:complete_graph}
For a complete graph ($r = n-1$), we have for GIN, $Q_{\mathrm{GIN}} N = \xi^{(k)} N$, i.e., graph structural information in $Q$ will be removed after multiplying $N$.
\end{proposition}
The proof of these two propositions can be found in Appendix~\ref{appsec:proofs}. Similar results can be easily derived for other architectures like GCN by substituting $Q_{\mathrm{GIN}}$ with $Q_{\mathrm{GCN}}$. As we can see from the above analysis, in graph data, the mean statistics after the aggregation sometimes contain structural information. Discarding the mean will degrade the expressiveness of the neural networks. Note that the problem may not happen in image domain. The mean statistics of image data contains global information such as brightness. Removing such information in images will not change the semantics of the objects and thus will not hurt the classification performance.

\begin{figure}[t]
    \centering
        \includegraphics[width=0.48\textwidth]{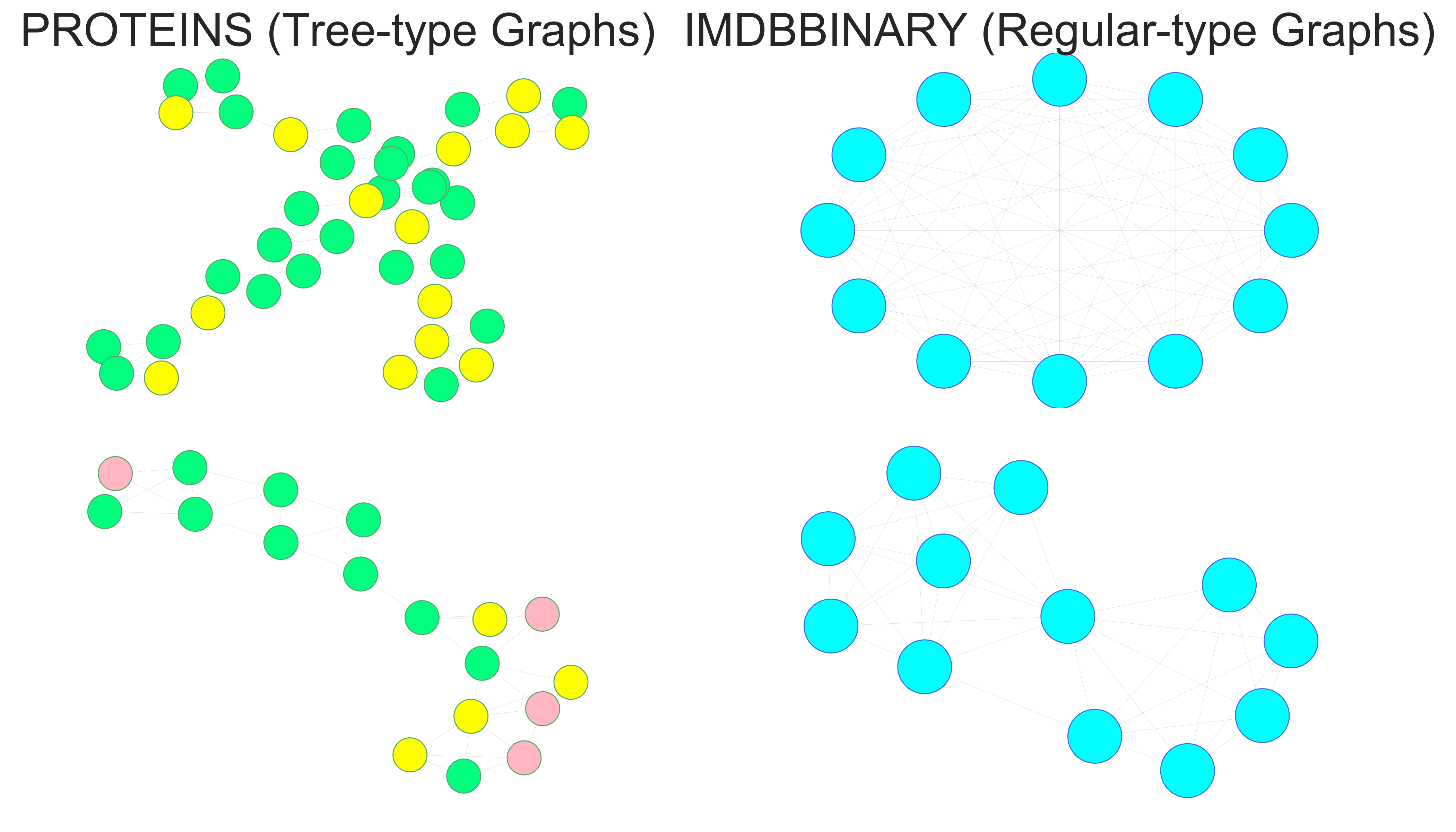}
        \includegraphics[width=0.48\textwidth]{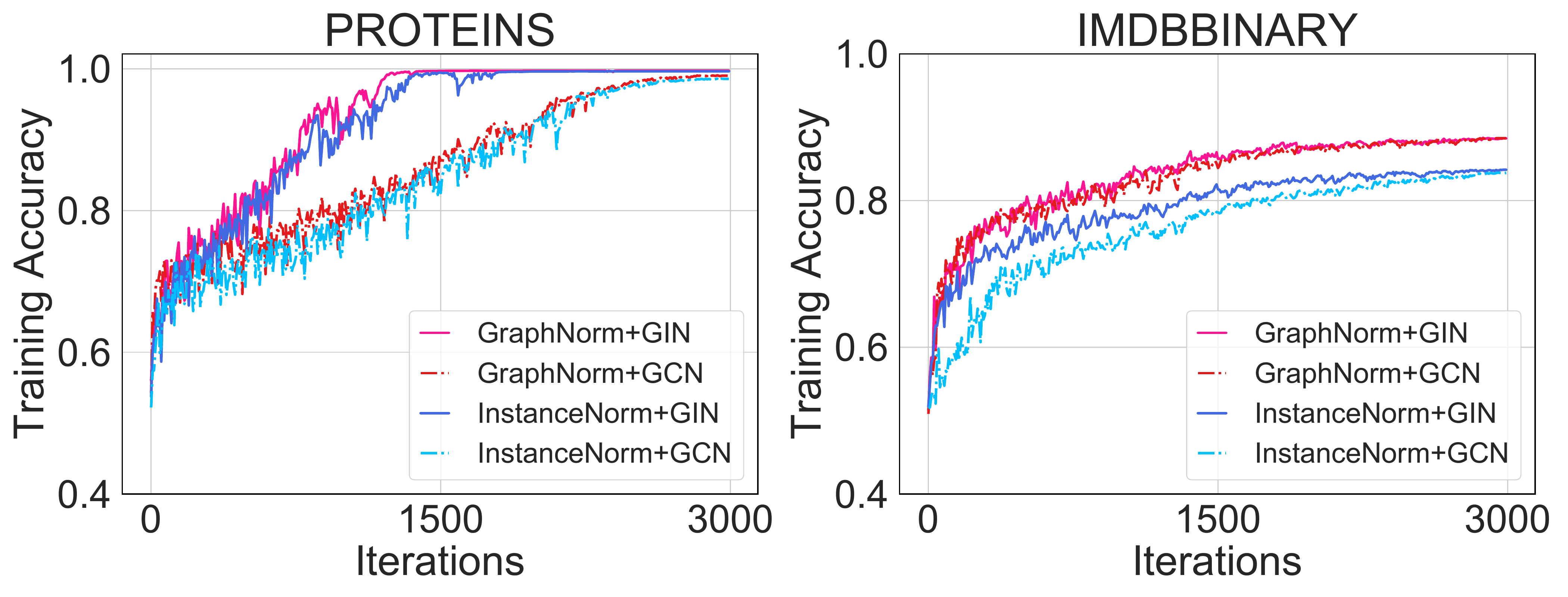}
    \caption{
    \textbf{Comparison of GraphNorm and InstanceNorm on different types of graphs.} Top: Sampled graphs with different topological structures. Bottom: Training curves of GIN/GCN using GraphNorm and InstanceNorm.}
\label{fig:ablation-study-learn-alpha}
\end{figure}

This analysis inspires us to modify the current normalization method with a \emph{learnable parameter} to automatically control how much the mean to preserve in the shift operation. Combined with the graph-wise normalization, we name our new method Graph Normalization, i.e., GraphNorm. For each graph $G$, we generally denote value $\hhat_{i,j}$ as the inputs to GraphNorm, e.g., the $j$-th feature value of node $v_i$, $i=1,\cdots,n$, $j=1,\cdots,d$. GraphNorm takes the following form:
\begin{align}
    \label{eq:graph_norm_formula}
    \mathrm{GraphNorm}\rbr{\hhat_{i,j}} = \gamma_j \cdot \frac{\hhat_{i,j} - \alpha_j\cdot\mu_j}{\hat{\sigma}_j} + \beta_j,
\end{align}
where $\mu_j=\frac{\sum_{i=1}^n\hhat_{i,j}}{n},\hat{\sigma}_j^2 = \frac{\sum_{i=1}^n\rbr{\hhat_{i,j}-\alpha_j\cdot\mu_j}^2}{n}$, and $\gamma_j, \beta_j$ are the affine parameters as in other normalization methods. By introducing the learnable parameter $\alpha_j$ for each feature dimension $j$, we are able to learn how much the information we need to keep in the mean. It is easy to see that GraphNorm has stronger expressive power than InstanceNorm. Formally, we have the following fact:
\begin{fact}[GraphNorm is strictly more expressive than InstanceNorm]\label{prop:express_compare}
    If $\alpha_j \neq 1, \gamma_j \neq 0$, then there does \textbf{not} exist $\gamma_j', \beta_j'$ such that for any $\cbr{\hhat_{i,j}}_{i=1}^n$ that the normalization is applied to, for any $i$,
        $\mathrm{GraphNorm}_{\cbr{\alpha_j, \gamma_j, \beta_j}}\rbr{\hhat_{i,j}} = \gamma_j \cdot \frac{\hhat_{i,j} - \alpha_j\cdot\mu_j}{\hat{\sigma}_j} + \beta_j = \gamma'_j \cdot \frac{\hhat_{i,j} - \mu_j}{\sigma_j} + \beta'_j = \mathrm{InstanceNorm}_{\cbr{\gamma'_j, \beta'_j}}\rbr{\hhat_{i,j}},$
    where $\mu_j=\frac{\sum_{i=1}^n\hhat_{i,j}}{n},\hat{\sigma}_j^2 = \frac{\sum_{i=1}^n\rbr{\hhat_{i,j}-\alpha_j\cdot\mu_j}^2}{n},\sigma_j^2 = \frac{\sum_{i=1}^n\rbr{\hhat_{i,j} - \mu_j}^2}{n}$.
\end{fact}

To validate our theory and the proposed GraphNorm in real-world data, we conduct an ablation study on two typical datasets, PROTEINS and IMDB-BINARY. As shown in Figure~\ref{fig:ablation-study-learn-alpha}, the graphs from PROTEINS and IMDB-BINARY exhibit irregular-type and regular-type graphs, respectively. We train GIN/GCN using our adaptation of InstanceNorm and GraphNorm under the same setting in Section~\ref{sec:exps}. The training curves are presented in Figure~\ref{fig:ablation-study-learn-alpha}. The curves show that using a learnable $\alpha$ slightly improves the convergence on PROTEINS, while significantly boost the training on IMDB-BINARY. This observation verify that shifting the feature values by subtracting the mean may lose information, especially for regular graphs. And the introduction of learnable shift in GraphNorm can effectively mitigate the expressive degradation.

\begin{table*}[t]
\caption{{\bf Test performance} of GIN/GCN with various normalization methods on graph classification tasks. }
\resizebox{\textwidth}{!}{ \renewcommand{\arraystretch}{1.25}
\begin{tabular}{@{}clcccccccc@{}}

\cmidrule[\heavyrulewidth]{2-9}

& Datasets &  {\textsc{MUTAG}} & {\textsc{PTC}}  & {\textsc{PROTEINS}} & {\textsc{NCI1}} & {\textsc{IMDB-B}} & {\textsc{RDT-B}} & {\textsc{COLLAB}}  \\

& \text{\# graphs }  & 188  & 344  & 1113  &  4110 & 1000 & 2000 & 5000     \\
& \text{\# classes }   &  2  & 2  & 2  &  2 &  2 &  2 & 2 \\
& \text{Avg \# nodes }  &  17.9  & 25.5  & 39.1  &  29.8 &  19.8 &  429.6 & 74.5 \\
\cmidrule[\heavyrulewidth]{2-9}
& \textsc{WL subtree \citep{shervashidze2011weisfeiler}}  & 90.4 $\pm$ 5.7 & 59.9 $\pm$ 4.3 & 75.0 $\pm$ 3.1 & \textbf{86.0 $\pm$ 1.8} & 73.8 $\pm$ 3.9 & 81.0 $\pm$ 3.1 & 78.9 $\pm$ 1.9 &   \\
& \textsc{DCNN \citep{atwood2016diffusion}}  &  67.0 &  56.6 &  61.3 &  62.6  &  49.1 &  - &  52.1 &   \\
& \textsc{DGCNN \citep{zhang2018end}}  &  85.8 & 58.6 & 75.5 & 74.4 & 70.0 & - & 73.7 &   \\
& \textsc{AWL \citep{ivanov2018anonymous}}  & 87.9 $\pm$ 9.8 & - & - & - & 74.5 $\pm$ 5.9 & 87.9 $\pm$ 2.5 & 73.9 $\pm$ 1.9 &   \\
\cmidrule[\heavyrulewidth]{2-9}

& \textsc{GIN+LayerNorm }  & 82.4 $\pm$ 6.4    &  62.8 $\pm$ 9.3  & 76.2 $\pm$ 3.0  & 78.3 $\pm$ 1,7 & 74.5 $\pm$ 4,4 & 82.8 $\pm$ 7.7  & 80.1 $\pm$ 0.8  &  \\
& \textsc{GIN+BatchNorm (\citep{xu2018how})}  &  89.4 $\pm$ 5.6    & 64.6 $\pm$ 7.0  & 76.2 $\pm$ 2.8  &  82.7 $\pm$ 1.7 & 75.1 $\pm$ 5.1 & 92.4 $\pm$ 2.5     &\textbf{ 80.2 $\pm$ 1.9}     &  \\

& \textsc{GIN+InstanceNorm }  & 90.5 $\pm$ 7.8    & 64.7 $\pm$ 5.9  & 76.5 $\pm$ 3.9 & 81.2 $\pm$ 1.8 & 74.8 $\pm$ 5.0 & 93.2 $\pm$ 1.7     & 80.0 $\pm$ 2.1    &  \\
& \textbf{GIN+GraphNorm}  &  \textbf{91.6 $\pm$ 6.5}  &  \textbf{ 64.9 $\pm$ 7.5} & \textbf{ 77.4 $\pm$ 4.9}  &  81.4 $\pm$ 2.4 & \textbf{76.0 $\pm$ 3.7} & \textbf{ 93.5 $\pm$ 2.1}     &\textbf{ 80.2 $\pm$ 1.0}     & \\
\cmidrule[\heavyrulewidth]{2-9}
\end{tabular}}
  \label{tab:test-results}
 
\end{table*}

\begin{table}[t]
    \centering
    \caption{\textbf{Test performance} on OGB.}
    
    \begin{tabular}{lc}
      \toprule
        \text{Datasets} &  {\textsc{Ogbg-molhiv}} \\
        \text{\# graphs }  & 41,127  \\
        \text{\# classes }   &  2  \\
        \text{Avg \# nodes }  &  25.5 \\
      \midrule
        \text{GCN \citep{hu2020open}} & 76.06 $\pm$ 0.97 \\
        \text{GIN \citep{hu2020open}} & 75.58 $\pm$ 1.40 \\
      \midrule
        \text{GCN+LayerNorm } & 75.04 $\pm$ 0.48 \\
        \text{GCN+BatchNorm } & 76.22 $\pm$ 0.95 \\
                \text{GCN+InstanceNorm } & 78.18 $\pm$ 0.42 \\
        \textbf{GCN+GraphNorm} & \textbf{78.30 $\pm$ 0.69} \\
      \midrule
        \text{GIN+LayerNorm } & 74.79 $\pm$ 0.92 \\
        \text{GIN+BatchNorm } & 76.61 $\pm$ 0.97 \\
              \text{GIN+InstanceNorm } & 77.54 $\pm$ 1.27 \\
        \textbf{GIN+GraphNorm} & \textbf{77.73 $\pm$ 1.29} \\
      \bottomrule
    \end{tabular}
    \label{tab:test-molhiv}
\end{table}

\section{Experiments}
\label{sec:exps}

In this section,  we evaluate and compare both the training and test performance of GraphNorm with other normalization methods on graph classification benchmarks. 

\paragraph{Settings.}
We use eight popularly used benchmark datasets of different scales in the experiments~\citep{yanardag2015deep,xu2018how}, including four medium-scale bioinformatics datasets (MUTAG, PTC, PROTEINS, NCI1), three medium-scale social network datasets (IMDB-BINARY, COLLAB, REDDIT-BINARY), and one large-scale bioinformatics dataset ogbg-molhiv, which is recently released on Open Graph Benchmark (OGB)~\citep{hu2020open}. Dataset statistics are summarized in Table \ref{tab:test-results}. We use two typical graph neural networks GIN~\citep{xu2018how} and GCN~\citep{kipf2016semi} for our evaluations. Specifically, we use a five-layer GCN/GIN. For GIN, the number of sub-layers in MLP is set to 2. Normalization is applied to each layer. To aggregate global features on top of the network, we use SUM readout for MUTAG, PTC, PROTEINS and NCI1 datasets, and use MEAN readout for other datasets, as in \citet{xu2018how}. Details of the experimental settings are presented in Appendix \ref{appsec:experiments}.
\paragraph{Results.}
\label{sec:exp-results}
We plot the training curves of GIN with GraphNorm and other normalization methods\footnote{The graph size normalization in the preliminary version of \citet{dwivedi2020benchmarking} does not show significant improvement on the training and test performance, so we do not report it.} on different tasks in Figure \ref{fig:gin-dataset-training-curve}. The results on GCN show similar trends, and are provided in Appendix~\ref{appsec:training_gcn}. As shown in Figure~\ref{fig:gin-dataset-training-curve}, GraphNorm enjoys the fastest convergence on all tasks. Compared to BatchNorm used in \citet{xu2018how}, GraphNorm converges in roughly 5000/500 iterations on NCI1 and PTC datasets, while the model using BatchNorm does not even converge in 10000/1000 iterations. Remarkably, though  InstanceNorm does \emph{not} outperform other normalization methods on IMDB-BINARY, GraphNorm with learnable shift significantly boosts the training upon InstanceNorm and achieves the fastest convergence. We also validate the test performance and report the test accuracy in Table~\ref{tab:test-results},\ref{tab:test-molhiv}. The results show that GraphNorm also improves the generalization on most benchmarks. 

For reference, we explain the possible reasons of higher test accuracy in two folds. First, as shown in Figure~\ref{fig:gin-dataset-training-curve}, using proper normalization helps the model find a minimum with a higher training accuracy. Second, as suggested by \citet{hardt2016train}, faster training leads to smaller generalization gap. Since the test accuracy equals the training accuracy plus the generalization, these two views together suggest better 
normalization leads to better test performance.

\subsection{Ablation Study}
\label{sec:ablation}
In this subsection, we summarize the results of some ablation studies, including BatchNorm with learnable shift, BatchNorm with running statistics and the effect of batch size. Due to the space limitation, the detailed results can be found in Appendix~\ref{appsec:additional_exps}.

\paragraph{BatchNorm with learnable shift.} We conduct experiments on BatchNorm to investigate whether simply introducing a learnable shift can already improve the existing normalization methods without concrete motivation of overcoming expressiveness degradation. Specifically, we equip BatchNorm with a similar learnable shift as GraphNorm and evaluate its performance. We find that the learnable shift cannot further improve upon BatchNorm (See Appendix ~\ref{appsec:additional_exps}), which suggests the introduction of learnable shift in GraphNorm is critical.

\paragraph{BatchNorm with running statistics.} We study the variant of BatchNorm which uses running statistics to replace the batch-level mean and standard deviation (Similar idea is also proposed in \citet{yan2019towards}). At first glance, this method may seem to be able to mitigate the problem of large batch noise. However, the running statistics change a lot during training, and using running statistics disables the model to back-propagate the gradients through mean and standard deviation. Results in Appendix ~\ref{appsec:additional_exps} show this variant has even worse performance than BatchNorm.

\paragraph{The effect of batch size.} 
We further compare the GraphNorm with BatchNorm with different batch sizes (8, 16, 32, 64). As shown in Appendix ~\ref{appsec:additional_exps}, our GraphNorm consistently outperforms the BatchNorm on all the settings.
\section{Conclusion and Future Work}
In this paper, we adapt and evaluate three well-used normalization methods, i.e., BatchNorm, LayerNorm, and InstanceNorm to GNNs. We give explanations for the successes and failures of these adaptations. Based on our understanding of the strengths and limitations of existing adaptations, we propose Graph Normalization, that builds upon the adaptation of InstanceNorm with a learnable shift to overcome the expressive degradation of the original InstanceNorm. Experimental results show GNNs with GraphNorm not only converge faster, but also achieve better generalization performance on several benchmark datasets.

Though seeking theoretical understanding of normalization methods in deep learning is challenging \citep{arora2018theoretical} due to limited understanding on the optimization of deep learning models and characterization of real world data, we take an initial step towards finding effective normalization methods for GNNs with theoretical guidance in this paper. The proposed theories and hypotheses are motivated by several simple models. And we are not able to give concrete theoretical results to problems such as: the convergence rate of general GNNs with normalization, the spectrum of $Q$ normalized by learnable shift, etc. We believe the analyses of more realistic but complicated settings, e.g., the dynamics of GraphNorm on deep GNNs, are good future directions.

\section*{Acknowledgements}
We thank Mozhi Zhang and Ruosong Wang for helpful suggestions on the paper; Zhiyuan Li and Kaifeng Lyu for helpful discussion on the literature of normalization methods; and Prof. Yang Yuan for support of computational resources. This work was supported by National Key R\&D Program of China (2018YFB1402600), Key-Area Research and Development Program of Guangdong Province
(No. 2019B121204008), BJNSF (L172037), Beijing
Academy of Artificial Intelligence, Project 2020BD006 supported by PKU-Baidu Fund, NSF CAREER award (1553284) and NSF III (1900933).

\bibliographystyle{icml2021}
\bibliography{iclr2021_conference}

\newpage
\appendix
\twocolumn
\section{Proofs}\label{appsec:proofs}
\subsection{Proof of Theorem~\ref{thm:precondition}}
\label{appsec:proof_precondition}
We first introduce the Cauchy interlace theorem:
\begin{lemma}[Cauchy interlace theorem~(Theorem 4.3.17 in \citet{horn2012matrix})]\label{lem:seperation}
    Let $S\in\Rbb^{(n-1)\times (n-1)}$ be symmetric, $y\in\Rbb^n$ and $a\in\Rbb$ be given, and let $R = \begin{pmatrix}
        S & y\\
        y^\top & a
    \end{pmatrix}\in\Rbb^{n\times n}$. Let $\lambda_1\le\lambda_2\le\cdots\le\lambda_n$ be the eigenvalues of $R$ and $\mu_1\le\mu_2\le\cdots\le\mu_{n-1}$ be the eigenvalues of $S$. Then
    \begin{align}
        \lambda_1\le\mu_1\le\lambda_2\le\cdots\le\lambda_{n-1}\le\mu_{n-1}\le\lambda_n,
    \end{align}
    where $\lambda_i = \mu_i$ only when there is a nonzero $z\in\Rbb^{n-1}$ such that $Sz = \mu_i z$ and $y^\top z = 0$; if $\lambda_i = \mu_{i-1}$ then there is a nonzero $z\in\Rbb^{n-1}$ such that $Sz = \mu_{i-1} z$, $y^\top z = 0$.
\end{lemma}
Using Lemma \ref{lem:seperation}, the theorem can be proved as below.
\begin{proof}
    For any matrices $P,R\in\Rbb^{n\times n}$, we use $P\sim R$ to denote that the matrix $P$ is similar to the matrix $R$. Note that if $P\sim R$, the eigenvalues of $P$ and $R$ are the same. As the singular values of $P$ are equal to the square root of the eigenvalues of $P^\top P$, we have the eigenvalues of $Q^\top Q$ and that of $NQ^\top{Q}N$ are $\cbr{\lambda_i^2}_{i=1}^n$ and $\cbr{\mu_i^2}_{i=1}^n$, respectively.

    Note that $N$ is a projection operator onto the orthogonal complement space of the subspace spanned by $\one$, and $N$ can be decomposed as $N = U\diag\rbr{\underbrace{1, \cdots, 1}_{\times n-1}, 0}U^\top$ where $U$ is an orthogonal matrix. Since $\one$ is the eigenvector of $N$ associated with eigenvalue $0$, we have
    \begin{align}
        U = \begin{pmatrix}
            U_1 & \frac{1}{\sqrt{n}}\one
        \end{pmatrix},
    \end{align}
    where $U_1\in\Rbb^{n\times (n-1)}$ satisfies $U_1\one=0$ and $U_1^\top U_1 = I_{n-1}$.

    Then we have $N{Q}^\top{Q} N = U\diag\rbr{1, \cdots, 1, 0}U^\top{Q}^\top{Q} U\diag\rbr{1, \cdots, 1, 0}U^\top\sim \diag\rbr{1, \cdots, 1, 0}U^\top{Q}^\top{Q} U\diag\rbr{1, \cdots, 1, 0}$.

    Let
    \begin{align}
        D &= \diag\rbr{1,\cdots,1,0}=
        \begin{pmatrix}
            I_{n-1} & \zero\\
            \zero^\top & 0
        \end{pmatrix},\\
        B &=
        \begin{pmatrix}
            I_{n-1}\\
            \zero^\top
        \end{pmatrix},\\
        \Cbar &= {Q}^\top{Q},
    \end{align}
    where $\zero = \sbr{\underbrace{0, \cdots, 0}_{\times n-1}}^\top$.
    
    We have
    \begin{align}
        N{Q}^\top{Q} N&\sim DU^\top\Cbar UD\\
        &= 
        D
        \begin{pmatrix}
            U_1^\top\\
            \frac{1}{\sqrt{n}}\one^\top
        \end{pmatrix}
        \Cbar
        \begin{pmatrix}
            U_1 & \frac{1}{\sqrt{n}}\one
        \end{pmatrix}
        D\\
        &=D
        \begin{pmatrix}
            U_1^\top\Cbar U_1 & \frac{1}{\sqrt{n}}U_1^\top\Cbar\one\\
            \frac{1}{\sqrt{n}}\one^\top\Cbar U_1 & \frac{1}{n}\one^\top\Cbar\one
        \end{pmatrix}
        D\\
        &=\begin{pmatrix}
            B^\top\\
            \begin{matrix}
                \zero^\top& 0
            \end{matrix}
        \end{pmatrix}
        \begin{pmatrix}
            U_1^\top\Cbar U_1 & \frac{1}{\sqrt{n}}U_1^\top\Cbar\one\\
            \frac{1}{\sqrt{n}}\one^\top\Cbar U_1 & \frac{1}{n}\one^\top\Cbar\one
        \end{pmatrix}
        \begin{pmatrix}
            B &\begin{matrix}
            \zero\\
            0
            \end{matrix}
        \end{pmatrix}\\
        &= \begin{pmatrix}
            U_1^\top\Cbar U_1 &\zero\\
            \zero^\top &0
        \end{pmatrix}\label{eq:block_decomp}.
    \end{align}
    Using Lemma~\ref{lem:seperation} and taking $R = U^\top\Cbar U$ and $S = U_1^\top\Cbar U_1$, we have the eigenvalues of $U_1^\top\Cbar U_1$ are interlacing between the eigenvalues of $U^\top\Cbar U$. Note that the eigenvalues of $DU^\top \Cbar UD$ are $\mu_1^2\le\mu_2^2\le\cdots\le\mu_{n-1}^2$ and $\mu_n^2=0$, and by Eq.~\eqref{eq:block_decomp}, the eigenvalues of $DU^\top \Cbar UD$ contain the eigenvalues of $U_1^\top \Cbar U_1$ and $0$. Since the eigenvalues of $U^\top \Cbar U$ are $\lambda_1^2\le\lambda_2^2\le\cdots\le\lambda_n^2$ (By similarity of $U^\top \Cbar U$ and $\Cbar$), we then have
    \begin{align}
        \lambda_1^2\le\mu_1^2\le\lambda_2^2\le\cdots\le\lambda_{n-1}^2\le\mu_{n-1}^2\le\lambda_n^2.
    \end{align}
    Moreover, the equality holds only when there is a nonzero $z\in\Rbb^{n-1}$ that satisfies
    \begin{align}
        U_1^\top\Cbar U_1 z = \mu z,\label{eq:equal_eigen}\\
        \one^\top\Cbar U_1 z = 0\label{eq:equal_otho},
    \end{align}
    where $\mu$ is one of $\mu_i^2$s. 
    
    Since $U_1$ forms an orthogonal basis of the orthogonal complement space of $\one$ and Eq.~\eqref{eq:equal_otho} is equivalent to ``$\Cbar U_1 z$ lies in the orthogonal complement space'', we have that there is a vector $y\in\Rbb^{n-1}$ such that
    \begin{align}
        \Cbar U_1 z = U_1 y.
    \end{align}
    Substituting this into Eq.~\eqref{eq:equal_eigen}, we have
    \begin{align}
        U_1^\top U_1 y = \mu z.
    \end{align}
    Since $U_1^\top U_1 = I_{n-1}$, the equation above is equivalent to
    \begin{align}
        y = \mu z,
    \end{align}
    which means
    \begin{align}
        \Cbar U_1 z = U_1 y = \mu U_1 z,
    \end{align}
    i.e., $U_1 z$ is the eigenvector of $\Cbar$ associated with $\mu$. By noticing $U_1 z$ lies in the orthogonal complement space of $\one$ and the eigenvector of $\Cbar$ is right singular vector of ${Q}$, we complete the proof.
\end{proof}

\subsection{Concrete example of the acceleration}
\label{appsec:concrete_example}
To get more intuition on how the preconditioning effect of the shift can accelerate the training of GNNs, we provide a concrete example showing that shift indeed improves the convergence rate. Note that the global convergence rate of widely-used deep GNNs on general data remains highly unexplored, and the existing works mainly focus on some simplified case, e.g., GNTK \citep{du2019graph}. To make things clear without loss of intuition, we focus on a \emph{simple linear GNN} applied to a \emph{well-specified task} where we are able to explicitly compare the convergence rates.

\subsubsection{Settings}
\paragraph{Data.}
We describe each sample, i.e., graph, with $n$ nodes by a tuple $G = \cbr{X, Q, \pbf, y}$, where
\begin{itemize}[leftmargin=*]
    \item $X\in\R^{d\times n}$ is the feature matrix of the graph, where $d$ is the dimension of the of each feature vector.
    \item $Q\in\R^{n\times n}$ representing the matrix representing the neighbor aggregation as Eq.~\eqref{eq:structure_with_norm}. Note that this matrix depends on the aggregation scheme used by the chosen architecture, but for simplicity, we model this as a part of data structure.
    \item $\pbf\in\R^{n\times 1}$ is a weight vector representing the importance of each node. This will be used to calculate the $\mathrm{READOUT}$ step. Note that this vector is not provided in many real-world datasets, so the $\mathrm{READOUT}$ step usually takes operations such as summation.
    \item $y\in\R$ is the label.
\end{itemize}
The whole dataset $S = \cbr{G_1, \cdots, G_m}$ consists of $m$ graphs where $G_i = \cbr{X_i, Q_i, \pbf_i, y_i}$. We make the following assumptions on the data generation process:
\begin{assumption}[Independency]\label{assumpt:indepency}
We assume $X_i$, $Q_i$, $\pbf_i$ are drawn from three independent distributions in an i.i.d. manner, e.g., $X_1, \cdots, X_m$ are i.i.d.. 
\end{assumption}
\begin{assumption}[Structure of data distributions]\label{assumpt:non-degeneracy}
For clearness and simplicity of statement, we assume the number of nodes in each graph $G_i$ are the same, we will use $n$ to denote this number and we further assume $n = d$. We assume that the distribution of $\pbf_i$ satisfies $\Ebb\sbr{\pbf\pbf^\top} = I_n, \Ebb \pbf = 0$, which means the importance vector is non-degenerate. Let $\Ebb XQ = Y$, we assume that $Y$ is \emph{full ranl}. We make the following assumptions on $XQ$: $\one^\top Y^{-1}XQ = 0$, which ensures that there is no information in the direction $\one^\top Y^{-1}$; there is a constant $\delta_1$ such that $\Ebb (XQ-Y)(XQ-Y)^\top \preceq \delta_1 I_d$ and $\Ebb (XQ-Y)N(XQ-Y)^\top\preceq \delta_1 I_d$, where $\delta_1$ characterizes the noise level; none of the eigenvectors of $YY^\top$ is orthogonal to $\one$.
\end{assumption}
\begin{remark}
A few remarks are in order, firstly, the assumption that each graph has the same number of nodes and the number $n$ is equal to feature dimension $d$ can be achieved by ``padding'', i.e., adding dummy points or features to the graph or the feature matrix. The assumption that $\one^\top Y^{-1}XQ = 0$ is used to guarantee that there is no information loss caused by shift ($\one^\top Y^{-1}YNY^\top = 0$). Though we make this strong assumption to ensure no information loss in theoretical part, we introduce ``learnable shift'' to mitigate this problem in the practical setting. The theory taking learnable shift into account is an interesting future direction.
\end{remark}
\begin{assumption}[Boundness]\label{assumpt:boundness}
We make the technical assumption that there is a constant $b$ such that the distributions of $X_i, Q_i, \pbf_i$ ensures
\begin{align}
\norm{X_i}\norm{Q_i}\norm{\pbf_i}\le\sqrt{b}.
\end{align}
\end{assumption}
\paragraph{Model.}
We consider a simple \emph{linear graph neural network} with parameter $\wbf\in\R^{d\times 1}$:
\begin{align}
    \label{eq:linear_model}
    f^{\mathrm{Vanilla}}_\wbf(X, Q, \pbf) = \wbf^\top X Q \pbf.
\end{align}
Then, the model with shift can be represented as:
\begin{align}
    \label{eq:linear_model_shift}
    f^{\mathrm{Shift}}_\wbf(X, Q, \pbf) = \wbf^\top X Q N \pbf,
\end{align}
where $N = I_n - \frac{1}{n}\one\one^\top$.
\paragraph{Criterion.}
We consider using square loss as training objective, i.e.,
\begin{align}
    \label{eq:linear_obj}
L(f) = \sum_{i = 1}^m\frac{1}{2}\rbr{f(X_i, Q_i, \pbf_i) - y_i}^2.
\end{align}
\paragraph{Algorithm.} We consider using gradient descent to optimize the objective function. Let the initial parameter $\wbf_0=0$. The update rule of $w$ from step $t$ to $t+1$ can be described as:
\begin{align}
    \wbf_{t+1} = \wbf_t - \eta \nabla_\wbf L(f_{\wbf_t}),
\end{align}
where $\eta$ is the learning rate.
\begin{theorem}
Under Assumption~\ref{assumpt:indepency},\ref{assumpt:non-degeneracy},\ref{assumpt:boundness}, for any $\epsilon > 0$ there exists constants $C_1, C_2$, such that for $\delta_1<C_1, m>C_2$, with probability $1-\epsilon$, the parameter $\wbf_t^{\mathrm{Vanilla}}$ of vanilla model converges to the optimal parameter $\wbf_*^{\mathrm{Vanilla}}$ linearly:
\begin{align}
    \norm{\wbf_t^{\mathrm{Vanilla}} - \wbf_*^{\mathrm{Vanilla}}}_2 \le O\rbr{\rho_1^t},
\end{align}
while the parameter $\wbf_t^{\mathrm{Shfit}}$ of the shifted model converges to the optimal parameter $\wbf_*^{\mathrm{Shfit}}$ linearly:
\begin{align}
    \norm{\wbf_t^{\mathrm{Shift}} - \wbf_*^{\mathrm{Shfit}}}_2 \le O\rbr{\rho_2^t},
\end{align}
where
\begin{align}
    1 > \rho_1 > \rho_2,
\end{align}
which indicates the shifted model has a faster convergence rate.
\end{theorem}
\begin{proof}
We firstly reformulate the optimization problem in matrix form. 

Notice that in our linear model, the representation and structure of a graph $G_i=\cbr{X_i, Q_i, \pbf_i, y_i}$ can be encoded as a whole in a single vector, i.e., $\zbf_i^{\mathrm{Vanilla}}=X_iQ_i\pbf_i\in\R^{d\times 1}$ for vanilla model in Eq.~\eqref{eq:linear_model}, and $\zbf_i^{\mathrm{Shift}}=X_iQ_iN\pbf_i\in\R^{d\times 1}$ for shifted model in Eq.~\eqref{eq:linear_model_shift}. We call $\zbf_i$ and $\zbf_i^{\mathrm{Shift}}$ ``combined features''. Let $Z^{\mathrm{Vanilla}} = \sbr{\zbf_1^{\mathrm{Vanilla}}, \cdots, \zbf_m^{\mathrm{Vanilla}}}\in\R^{d\times m}$ and $Z^{\mathrm{Shift}} = \sbr{\zbf_1^{\mathrm{Shift}}, \cdots, \zbf_m^{\mathrm{Shift}}}\in\R^{d\times m}$ be the matrix of combined features of valinna linear model and shifted linear model respectively. For clearness of the proof, we may abuse the notations and use $Z$ to represent $Z^{\mathrm{Vanilla}}$. Then the objective in Eq.~\eqref{eq:linear_obj} for vanilla linear model can be reformulated as:
\begin{align}
\label{eq:linear_obj_mat}
L(f_\wbf) = \frac{1}{2}\norm{Z^\top \wbf - \ybf}^2_2,
\end{align}
where $\ybf = \sbr{y_1, \cdots, y_m}^\top\in\R^{m\times 1}$.

Then the gradient descent update can be explicitly writen as:
\begin{align}
\label{eq:linear_mat_update}
\wbf_{t+1} &= \wbf_t - \eta \rbr{ZZ^\top \wbf_t - Z\ybf}\\
&= (I_d - \eta ZZ^\top) \wbf_t + \eta Z\ybf,
\end{align}
which converges to $\wbf_* = \rbr{ZZ^\top}^\dagger Z\ybf$ according to classic theory of least square problem \citep{horn2012matrix}, where $\rbr{ZZ^\top}^\dagger$ is the Moore–Penrose inverse of $ZZ^\top$.

By simultaneously subtracting $\wbf_*$ in the update rule, we have 
\begin{align}
    \wbf_{t+1} - \wbf_* = \rbr{I_d - \eta ZZ^\top} \rbr{\wbf_t - \wbf_*}.
\end{align}

So the residual of $\wbf_t$ is
\begin{align}
    \norm{\wbf_t - \wbf_*} &= \norm{\rbr{I_d - \eta ZZ^\top}^t \wbf_*}\\
    &\le \norm{I_d - \eta ZZ^\top}^t\norm{\wbf_*}.\label{eq:linear_converge}
\end{align}
Let $\sigma_{\max}(A)$ and $\sigma_{\min}(A)$ be the maximal and mininal \emph{positive} eigenvalues of $A$, respectively. Then the optimial learning rate (the largest learning rate that ensures $I_d - \eta ZZ^\top$ is positive semidefinite) is $\eta = \frac{1}{\sigma_{\max}(ZZ^\top)}$. Under this learning rate we have the convergence rate following Eq.~\eqref{eq:linear_converge}:
\begin{align}
    \norm{\wbf_t - \wbf_*} &\le \norm{I_d - \eta ZZ^\top}^t\norm{\wbf_*}\\
    &\le \rbr{1 - \frac{\sigma_{\min}\rbr{ZZ^\top}}{\sigma_{\max}\rbr{ZZ^\top}}}^t\norm{\wbf_*}.\label{eq:linear_rate}
\end{align}

For now, we show that the convergence rate of the optimization problem with vanilla model depends on $\frac{\sigma_{\min}\rbr{ZZ^\top}}{\sigma_{\max}\rbr{ZZ^\top}}$. Follwing the same argument, we can show the convergence rate of the optimization problem with shifted model depends on $\frac{\sigma_{\min}\rbr{Z^{\mathrm{Shift}}Z^{\mathrm{Shfit}\top}}}{\sigma_{\max}\rbr{Z^{\mathrm{Shift}}Z^{\mathrm{Shfit}\top}}}$. We then aim to bound this term, which we call effective condition number.

Similarly, we investigate the effective condition number for $ZZ^\top$ first, and the analysis of $Z^{\mathrm{Shift}}Z^{\mathrm{Shift}\top}$ follows the same manner. As multiplying a constant does not affect the effective condition number, we first scale $ZZ^\top$ by $\frac{1}{m}$ and expand it as:
\begin{align}
    \frac{1}{m} ZZ^\top &= \frac{1}{m}\sum_{i=1}^m \zbf_i\zbf_i^\top,
\end{align}
which is the empirical estimation of the covariance matrix of the combined feature. By concentration inequality, we know this quantity is concentrated to the covariance matrix, i.e.,
\begin{align*}
    \Ebb_{\zbf} \zbf\zbf^\top &= \Ebb_{X,Q,\pbf}{XQ\pbf\rbr{XQ\pbf}^\top}\\
    &=\Ebb_{X,Q} XQ\rbr{\Ebb\sbr{\pbf\pbf^\top}}(XQ)^\top\\
    &=\Ebb_{X, Q} XQ(XQ)^\top\hspace{0.5cm}{\text{(By Assumption~\ref{assumpt:indepency})}}\\
    &=YY^\top + \Ebb_{X,Q} (XQ-Y)(XQ-Y)^\top.
\end{align*}
Noticing that $\zero \preceq \Ebb_{X,Q} (XQ-Y)(XQ-Y)^\top \preceq \delta_1 I_d$ by Assumption~\ref{assumpt:non-degeneracy}, and $Y$ is full rank, we can conclude that $\sigma_{\max}\rbr{YY^\top}\le\sigma_{\max}\rbr{\Ebb_{\zbf} \zbf\zbf^\top}\le\sigma_{\max}\rbr{YY^\top}+\delta_1$, and $\sigma_{\min}\rbr{YY^\top}\le\sigma_{\min}\rbr{\Ebb_{\zbf} \zbf\zbf^\top}\le\sigma_{\min}\rbr{YY^\top}+\delta_1$ by Weyl's inequality.

By similar argument, we have that $\frac{1}{m}Z^{\mathrm{Shift}}Z^{\mathrm{Shift}\top}$ concentrates to
\begin{align*}
    &\Ebb_{\zbf^{\mathrm{Shift}}} \zbf^{\mathrm{Shift}}\zbf^{\mathrm{Shift}\top} \\
    =& \Ebb_{X, Q} (XQ)N^2(XQ)^\top\\
    =&\Ebb_{X, Q} (XQ)N(XQ)^\top \hspace{0.5cm}{(N^2 = N)}\\
    =&YNY^\top + \Ebb_{X, Q} (XQ-Y)N(XQ-Y)^\top.
\end{align*}
By Assumption~\ref{assumpt:non-degeneracy}, we have
\begin{align*}
    0=&\one^\top Y^{-1}\Ebb_{\zbf^{\mathrm{Shift}}} \zbf^{\mathrm{Shift}}\zbf^{\mathrm{Shift}\top}\\
    =&\one^\top Y^{-1}\rbr{YNY^\top + \Ebb_{X, Q} (XQ-Y)N(XQ-Y)^\top}\\
    =&\one^\top Y^{-1}\Ebb_{X, Q} (XQ-Y)N(XQ-Y)^\top,
\end{align*}
which means $\Ebb_{X, Q} (XQ-Y)N(XQ-Y)^\top$ has the same eigenspace as $YNY^\top$ with respect to eigenvalue $0$. Combining with $\zero \preceq \Ebb_{X,Q} (XQ-Y)N(XQ-Y)^\top \preceq \delta_1 I_d$, we have $\sigma_{\max}\rbr{YNY^\top}\le\sigma_{\max}\rbr{\Ebb_{\zbf^{\mathrm{Shift}}} \zbf^{\mathrm{Shift}}\zbf^{\mathrm{Shift}\top}}\le\sigma_{\max}\rbr{YNY^\top}+\delta_1$, and $\sigma_{\min}\rbr{YNY^\top}\le\sigma_{\min}\rbr{\Ebb_{\zbf^{\mathrm{Shift}}} \zbf^{\mathrm{Shift}}\zbf^{\mathrm{Shift}\top}}\le\sigma_{\min}\rbr{YNY^\top}+\delta_1$.

It remains to bound the finite sample error, i.e., $\norm{\frac{1}{m}ZZ^\top - \Ebb_\zbf \zbf\zbf^\top}_2$ and $\norm{\frac{1}{m}Z^{\mathrm{Shift}}Z^{\mathrm{Shfit}\top} - \Ebb_\zbf \zbf\zbf^\top}_2$. These bounds can be obtained by the following lemma:
\begin{lemma}[Corollary 6.1 in \citet{wainwright2019high}]
Let $\zbf_1, \cdots, \zbf_m$ be i.i.d. zero-mean random vectors with covariance matrix $\Sigma$ such that $\norm{\zbf}_2\le\sqrt{b}$ almost surely. Then for all $\delta>0$, the sample covariance matrix $\hat{\Sigma} = \frac{1}{m}\sum_{i=1}^m\zbf_i\zbf_i^\top$ satisfies
\begin{align}
\Pr\sbr{\norm{\hat{\Sigma} - \Sigma}_2\ge\delta}\le 2d\exp\rbr{-\frac{\delta^2}{2b\rbr{\norm{\Sigma}_2+\delta}}}.
\end{align}
\end{lemma}
By this lemma, we further have
\begin{lemma}[Bound on the sample covariance matrix]
Let $\zbf_1, \cdots, \zbf_m$ be i.i.d. zero-mean random vectors with covariance matrix $\Sigma$ such that $\norm{\zbf}_2\le\sqrt{b}$ almost surely. Then with probability $1-\epsilon$, the sample covariance matrix $\hat{\Sigma} = \frac{1}{m}\sum_{i=1}^m\zbf_i\zbf_i^\top$ satisfies
\begin{align}
\norm{\hat{\Sigma} - \Sigma}_2\le O\rbr{\sqrt{\frac{\log (1/\epsilon)}{m}}},
\end{align}
where we hide constants $b, \norm{\Sigma}_2, d$ in the big-O notation and highlight the dependence on the number of samples $m$.
\end{lemma}
Combining with previous results, we conclude that:
\begin{align*}
&\sigma_{\max}\rbr{YY^\top}-O\rbr{\sqrt{\frac{\log (1/\epsilon)}{m}}}\\
\le&\sigma_{\max}\rbr{\frac{1}{m}ZZ^\top}\\
\le&\sigma_{\max}\rbr{YY^\top}+\delta_1+O\rbr{\sqrt{\frac{\log (1/\epsilon)}{m}}};
\end{align*}
\begin{align*}
&\sigma_{\min}\rbr{YY^\top}-O\rbr{\sqrt{\frac{\log (1/\epsilon)}{m}}}\\
\le&\sigma_{\min}\rbr{\frac{1}{m}ZZ^\top}\\
\le&\sigma_{\min}\rbr{YY^\top}+\delta_1+O\rbr{\sqrt{\frac{\log (1/\epsilon)}{m}}};
\end{align*}
\begin{align*}
&\sigma_{\max}\rbr{YNY^\top}-O\rbr{\sqrt{\frac{\log (1/\epsilon)}{m}}}\\
\le&\sigma_{\max}\rbr{\frac{1}{m}Z^{\mathrm{Shift}}Z^{\mathrm{Shift}\top}}\\
\le&\sigma_{\max}\rbr{YNY^\top}+\delta_1+O\rbr{\sqrt{\frac{\log (1/\epsilon)}{m}}}
\end{align*}
\begin{align*}
&\sigma_{\min}\rbr{YNY^\top}-O\rbr{\sqrt{\frac{\log (1/\epsilon)}{m}}}\\
\le&\sigma_{\min}\rbr{\frac{1}{m}Z^{\mathrm{Shift}}Z^{\mathrm{Shift}\top}}\\
\le&\sigma_{\min}\rbr{YNY^\top}+\delta_1+O\rbr{\sqrt{\frac{\log (1/\epsilon)}{m}}}.
\end{align*}

By now, we have transfered the analysis of $ZZ^\top$ and $Z^{\mathrm{Shift}}Z^{\mathrm{Shfit}\top}$ to the analysis of $YY^\top$ and $YNY^\top$. And the positive eigenvalues of $YNY^\top$ is interlaced between the positive eigenvalues of $YY^\top$ by the same argument as Theorem~\ref{thm:precondition}. Concretely, we have $\sigma_{\min}\rbr{YY^\top}\le\sigma_{\min}\rbr{YNY^\top}\le\sigma_{\max}\rbr{YNY^\top}\le\sigma_{\max}\rbr{YY^\top}$. Noticing that none of the eigenvectors of $YY^\top$ is orthogonal to $\one$, the first and last equalies can not be achieved, so $\sigma_{\min}\rbr{YY^\top}<\sigma_{\min}\rbr{YNY^\top}\le\sigma_{\max}\rbr{YNY^\top}<\sigma_{\max}\rbr{YY^\top}$. Finally, we can conclude for small enough $\delta_1$ and large enough $m$, with probability $\epsilon$,
\begin{align*}
&\sigma_{\min}\rbr{\frac{1}{m}ZZ^\top}\\
\le&\sigma_{\min}\rbr{YY^\top}+\delta_1+O\rbr{\sqrt{\frac{\log (1/\epsilon)}{m}}}\\
<&\sigma_{\min}\rbr{YNY^\top}-O\rbr{\sqrt{\frac{\log (1/\epsilon)}{m}}}\\
\le&\sigma_{\min}\rbr{\frac{1}{m}Z^{\mathrm{Shift}}Z^{\mathrm{Shift}\top}}\\
\le&\sigma_{\max}\rbr{\frac{1}{m}Z^{\mathrm{Shift}}Z^{\mathrm{Shift}\top}}\\
\le&\sigma_{\max}\rbr{YNY^\top}+\delta_1+O\rbr{\sqrt{\frac{\log (1/\epsilon)}{m}}}\\
<&\sigma_{\max}\rbr{YY^\top}-O\rbr{\sqrt{\frac{\log (1/\epsilon)}{m}}}\\
\le&\sigma_{\max}\rbr{\frac{1}{m}ZZ^\top}.
\end{align*}
So
\begin{align*}
&\rho_2 = 1 - \frac{\sigma_{\min}\rbr{Z^{\mathrm{Shift}}Z^{\mathrm{Shift}\top}}}{\sigma_{\max}\rbr{Z^{\mathrm{Shift}}Z^{\mathrm{Shift}\top}}}\\
<&\rho_1 = 1 - \frac{\sigma_{\min}\rbr{ZZ^\top}}{\sigma_{\max}\rbr{ZZ^\top}},
\end{align*}
where $\rho_1,\rho_2$ are the constants in the statement of the theorem. This inequality means the shifted model has better convergence speed by Eq.~\eqref{eq:linear_rate}.
\end{proof}

\subsection{Proof of Proposition~\ref{prop:regular_graph}}
\begin{proof}
    For $r$-regular graph, $A=r\cdot I_n$ and ${Q_{\mathrm{GIN}}} = \rbr{r+1+\xi^{(1)}}I_n$. Since $H^{(0)}$ is given by one-hot encodings of node degrees, the row of $H^{(0)}$ can be represented as $c\cdot \one^\top$ where $c=1$ for the $r$-th row and $c=0$ for other rows. By the associative property of matrix multiplication, we only need to show $H^{(0)}{Q_{\mathrm{GIN}}} N=0$. This is because, for each row
    \begin{align}
        c\cdot\one^\top {Q_{\mathrm{GIN}}} N &= c\cdot\one^\top (r+1+\xi^{(1)})I_n \rbr{I_n - \frac{1}{n}\one\one^\top}\\
        &=c\rbr{r+1+\xi^{(1)}}\rbr{\one^\top -\one^\top\cdot\frac{1}{n}\one\one^\top} = 0.
    \end{align}
\end{proof}
\subsection{Proof of Proposition~\ref{prop:complete_graph}}
\begin{proof}
\begin{align}
    {Q_{\mathrm{GIN}}} N= (A+ I_n+\xi^{(k)}I_n)N = = (\one\one^\top +\xi^{(k)I_n})N = \xi^{(k)}N,
\end{align}
\end{proof}
\subsection{Gradient of $W^{(k)}$}\label{appsec:gradient_formular}
We first calculate the gradient of $W^{(k)}$ when using normalization. Denote $Z^{(k)} = \mathrm{Norm}\rbr{W^{(k)}H^{(k-1)}Q}$ and $\Lcal$ as the loss.
Then the gradient of $\Lcal$ w.r.t. the weight matrix $W^{(k)}$ is 
\begin{align}
    \frac{\partial\Lcal}{\partial W^{(k)}}=
    \rbr{\rbr{H^{(k-1)}QN}^\top\otimes S}\frac{\partial\Lcal}{\partial Z^{(k)}},
\end{align}
where $\otimes$ represents the Kronecker product, and thus $\rbr{H^{(k-1)}QN}^\top\otimes S$ is an operator on matrices.

Analogously, the gradient of $W^{(k)}$ without normalization consists a $\rbr{H^{(k-1)}Q}^\top\otimes I_n$ term. As suggested by Theorem~\ref{thm:precondition}, $QN$ has a smoother distribution of spectrum than $Q$, so that the gradient of $W^{(k)}$ with normalization enjoys better optimization curvature than that without normalizaiton.

\section{Datasets}\label{appsec:datasets}

Detailed of the datasets used in our experiments are presented in this section. Brief statistics of the datasets are summarized in Table \ref{tab:appendix-dataset-statistics}. Those information can be also found in \citet{xu2018how} and \citet{hu2020open}.

\paragraph{Bioinformatics datasets.}
PROTEINS is a dataset where nodes are secondary structure elements (SSEs) and there is an edge between two nodes if they are neighbors in the amino-acid sequence or in 3D space. It has 3 discrete labels, representing helix, sheet or turn.
NCI1 is a dataset made publicly available by the National Cancer Institute (NCI) and is a subset of balanced datasets of chemical compounds screened for ability to suppress or inhibit the growth of a panel of human tumor cell lines, having 37 discrete labels.
MUTAG is a dataset of 188 mutagenic aromatic and heteroaromatic nitro compounds with 7 discrete labels. 
PTC is a dataset of 344 chemical compounds that reports the carcinogenicity for male and female rats and it has 19 discrete labels.

\paragraph{Social networks datasets.}
IMDB-BINARY is a movie collaboration dataset. Each graph corresponds to an ego-network for each actor/actress, where nodes correspond to actors/actresses and an edge is drawn betwen two actors/actresses if they appear in the same movie. Each graph is derived from a pre-specified genre of movies, and the task is to classify the genre graph it is derived from.
REDDIT-BINARY is a balanced dataset where each graph corresponds to an online discussion thread and nodes correspond to users. An edge was drawn between two nodes if at least one of them responded to another's comment.
The task is to classify each graph to a community or a subreddit it belongs to.
COLLAB is a scientific collaboration dataset, derived from 3 public collaboration datasets, namely, High Energy Physics, Condensed Matter
Physics and Astro Physics. Each graph corresponds to an ego-network of different researchers from
each field. The task is to classify each graph to a field the corresponding researcher belongs to.

\paragraph{Large-scale Open Graph Benchmark: ogbg-molhiv.}
Ogbg-molhiv is a molecular property prediction dataset, which is adopted from the the MOLECULENET \citep{DBLP:journals/corr/WuRFGGPLP17}. Each graph represents a molecule, where nodes are atoms and edges are chemical bonds. Both nodes and edges have associated diverse features. Node features are 9-dimensional, containing atomic number and chirality, as well as other additional atom features. Edge features are 3-dimensional, containing bond type, stereochemistry as well as an additional bond feature indicating whether the bond is conjugated. 

\begin{table*}[t]
\caption{{\bf Summary of statistics of benchmark datasets.}}
\resizebox{\textwidth}{!}{ \renewcommand{\arraystretch}{1.25}
\begin{tabular}{@{}clcccccccc@{}}
\cmidrule[\heavyrulewidth]{2-10}
& Datasets &  {\textsc{MUTAG}} & {\textsc{PTC}}  & {\textsc{PROTEINS}} & {\textsc{NCI1}} & {\textsc{IMDB-B}} & {\textsc{RDT-B}} & {\textsc{COLLAB}} & {\textsc{ogbg-molhiv}}  \\
\cmidrule[\heavyrulewidth]{2-10}
& \text{\# graphs }  & 188  & 344  & 1113  &  4110 & 1000 & 2000 & 5000 & 41127   \\
& \text{\# classes }   &  2  & 2  & 2  &  2 &  2 &  2 & 2 & 2 \\
& \text{Avg \# nodes }  &  17.9  & 25.5  & 39.1  &  29.8 &  19.8 &  429.6 & 74.5 & 25.5 \\
& \text{Avg \# edges }  &  57.5  & 72.5  & 184.7  &  94.5 &  212.8 &  1425.1 & 4989.5 & 27.5 \\
& \text{Avg \# degrees } & 3.2 & 3.0 & 4.7 & 3.1 & 10.7 & 3.3 & 66.9 & 2.1 \\
\cmidrule[\heavyrulewidth]{2-10}
\end{tabular}}
  \label{tab:appendix-dataset-statistics}
\end{table*}

\section{The Experimental Setup}\label{appsec:experiments}

\paragraph{Network architecture.}
For the medium-scale bioinformatics and social network datasets, we use 5-layer GIN/GCN with a linear output head for prediction followed \citet{xu2018how} with residual connection. The hidden dimension of GIN/GCN is set to be 64. For the large-scale ogbg-molhiv dataset, we also use 5-layer GIN/GCN\citep{xu2018how} architecture with residual connection. Following \citet{hu2020open}, we set the hidden dimension as 300.  

\paragraph{Baselines.} For the medium-scale bioinformatics and social network datasets, we compare several competitive baselines as in \citet{xu2018how}, including the WL subtree kernel model \citep{shervashidze2011weisfeiler}, diffusion-convolutional neural networks (DCNN)~\citep{atwood2016diffusion}, Deep Graph CNN (DGCNN) \citep{zhang2018end} and Anonymous Walk Embeddings (AWL) \citep{ivanov2018anonymous}. We report the accuracies reported in the original paper \citep{xu2018how}. For the large-scale ogbg-molhiv dataset, we use the baselines in \citet{hu2020open}, including the Graph-agnostic MLP model,  GCN \citep{kipf2016semi} and GIN \citep{xu2018how}. We also report the roc-auc values reported in the original paper \citep{hu2020open}.

\paragraph{Hyper-parameter configurations.}
We use Adam \citep{kingma2014adam} optimizer with a linear learning rate decay schedule. 
We follow previous work \citet{xu2018how} and \citet{hu2020open} to use hyper-parameter search (grid search) to select the best hyper-parameter based on validation performance. In particular, we select the batch size $\in\{64, 128\}$, the dropout ratio $\in\{0,0.5\}$, weight decay $\in\{5e-2, 5e-3, 5e-4,5e-5\}\cup\{0.0\}$,  the learning rate $\in\{1e-4, 1e-3, 1e-2\}$. For the drawing of the training curves in Figure \ref{fig:gin-dataset-training-curve}, for simplicity, we set batch size to be 128, dropout ratio to be 0.5, weight decay to be 0.0, learning rate to be 1e-2, and train the models for 400 epochs for all settings. 

\paragraph{Evaluation.} Using the chosen hyper-parameter, we report the averaged test performance over different random seeds (or cross-validation). In detail, for the medium-scale datasets, following \citet{xu2018how}, we perform a 10-fold cross-validation as these datasets do not have a clear train-validate-test splitting format. The mean and standard deviation of the validation accuracies across the 10 folds are reported. For the ogbg-molhiv dataset, we follow the official setting~\citep{hu2020open}. We repeat the training process with 10 different random seeds. 

For all experiments, we select the best model checkpoint with the best validation accuracy and record the corresponding test performance. 

\section{Additional Experimental Results}\label{appsec:additional_exps}

\subsection{Visualization of the singular value distributions}\label{appsec:vis-singular}
As stated in Theorem \ref{thm:precondition}, the shift operation $N$ serves as a preconditioner of ${Q}$ which makes the singular value distribution of ${Q}$ smoother. To check the improvements, we sample graphs from 6 median-scale datasets (PROTEINS, NCI1, MUTAG, PTC, IMDB-BINARY, COLLAB) for visualization, as in Figure \ref{fig:appendix-singular-value}.

\subsection{Visualization of noise in the batch statistics }\label{appsec:vis-noise}
We show the noise of the batch statistics on the PROTEINS task in the main body. Here we provide more experiment details and results. 

For graph tasks (PROTEINS, PTC, NCI1, MUTAG, IMDB-BINARY datasets), we train a 5-layer GIN with BatchNorm as in \citet{xu2018how} and the number of sub-layers in MLP is set to 2. For image task (CIFAR10 dataset), we train a ResNet18 \citep{he2016deep}. Note that for a 5-layer GIN model, it has four graph convolution layers (indexed from 0 to 3) and each graph convolution layer has two BatchNorm layers; for a ResNet18 model, except for the first 3$\times$3 convolution layer and the final linear prediction layer, it has four basic layers (indexed from 0 to 3) and each layer consists of two basic blocks (each block has two BatchNorm layers). For image task, we set the batch size as 128, epoch as 100, learning rate as 0.1 with momentum 0.9 and weight decay as 5e-4. For graph tasks, we follow the setting of Figure \ref{fig:gin-dataset-training-curve} (described in Appendix \ref{appsec:experiments}). 

The visualization of the noise in the batch statistics is obtained as follows. We first train the models and dump the model checkpoints at the end of each epoch; Then we randomly sample one feature dimension and fix it. For each model checkpoint, we feed different batches to the model and record the maximum/minimum batch-level statistics (mean and standard deviation) of the feature dimension across different batches. We also calculate dataset-level statistics. 

As Figure \ref{fig:noise-comparison} in the main body, pink line denotes the dataset-level statistics, and green/blue line denotes the maximum/minimum value of the batch-level statistics respectively. First, we provide more results on PTC, NCI1, MUTAG, IMDB-BINARY tasks, as in Figure \ref{fig:appendix-noise-in-graphs-other-dts}. We visualize the statistics from the first (layer-0) and the last (layer-3) BatchNorm layers in GIN for comparison. Second, we further visualize the statistics from different BatchNorm layers (layer 0 to layer 3) in GIN on PROTEINS and ResNet18 in CIFAR10, as in Figure \ref{fig:appendix-noise-in-graphs-layers}. Third, we conduct experiments to investigate the influence of the batch size. We visualize the statistics from BatchNorm layers under different settings of batch sizes [8, 16, 32, 64], as in Figure \ref{fig:appendix-noise-in-graphs-batch-sizes}. We can see that the observations are consistent and the batch statistics on graph data are noisy, as in Figure \ref{fig:noise-comparison} in the main body. 

\subsection{Training Curves on GCN}\label{appsec:training_gcn}
As Figure 2 in the main body, we train GCNs with different normalization methods (GraphNorm, InstanceNorm, BatchNorm and LayerNorm) and GCN without normalization in graph classification tasks and plot the training curves in Figure 6. It is obvious that the GraphNorm also enjoys the fastest convergence on all tasks. Remarkably, GCN with InstanceNorm even underperforms GCNs with other normalizations, while our GraphNorm with learnable shift significantly boosts the training upon InstanceNorm and achieves the fastest convergence.

\subsection{Further Results of Ablation Study}\label{appsec:ablation_study}
\paragraph{BatchNorm with learnable shift.} We conduct experiments on BatchNorm to investigate whether simply introducing a learnable shift can already improve the existing normalization methods without concrete motivation of overcoming expressiveness degradation. Specifically, we equip BatchNorm with a similar learnable shift ($\alpha$-BatchNorm for short) as GraphNorm and evaluate its performance. As shown in Figure 12, the $\alpha$-BatchNorm cannot outperform the BatchNorm on the three datasets. Moreover, as shown in Figure 5 in the main body, the learnable shift significantly improve upon GraphNorm on IMDB-BINARY dataset, while it cannot further improve upon BatchNorm, which suggests the introduction of learnable shift in GraphNorm is critical.

\paragraph{BatchNorm with running statistics.} We study the variant of BatchNorm which uses running statistics (MS-BatchNorm for short) to replace the batch-level mean and standard deviation (similar idea is also proposed in \citet{yan2019towards}). At first glance, this method may seem to be able to mitigate the problem of large batch noise. However, the running statistics change a lot during training, and using running statistics disables the model to back-propagate the gradients through mean and standard deviation. Thus, we also train GIN with BatchNorm which stops the back-propagation of the graidients through mean and standard deviation (DT-BatchNorm for short). As shown in Figure 12, both the MS-BatchNorm and DT-BatchNorm underperform the BatchNorm by a large margin, which shows that the problem of the heavy batch noise cannot be mitigated by simply using the running statistics.

\paragraph{The effect of batch size.} 
We further compare the GraphNorm and BatchNorm with different batch sizes (8, 16, 32, 64). As shown in Figure 11, our GraphNorm consistently outperforms the BatchNorm on all the settings.

\begin{figure*}[ht]
    \centering
        \includegraphics[width=\textwidth]{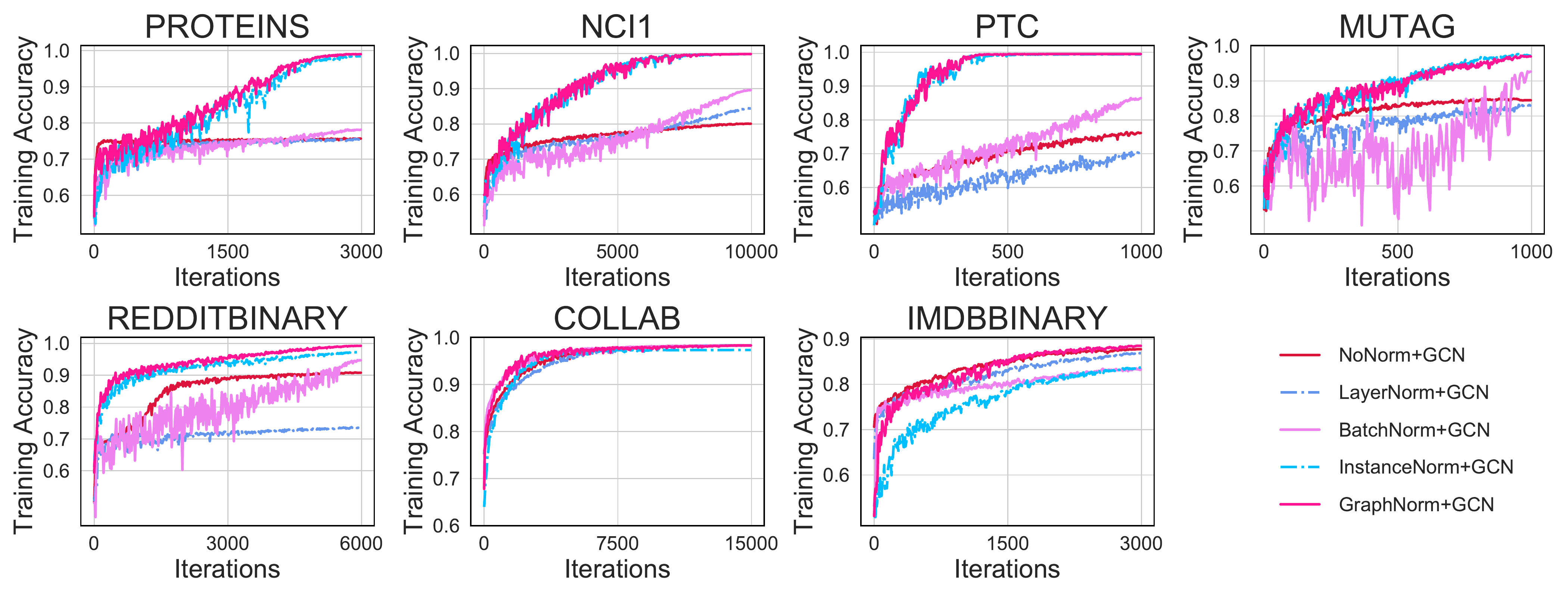}
    \caption{{\bf Training performance} of GCN with different normalization methods and GCN without normalization in graph classification tasks. }
\label{fig:appendix-norm-GCN-ablation}
\end{figure*}

\begin{figure*}[ht]
    \centering        \includegraphics[width=0.8\textwidth]{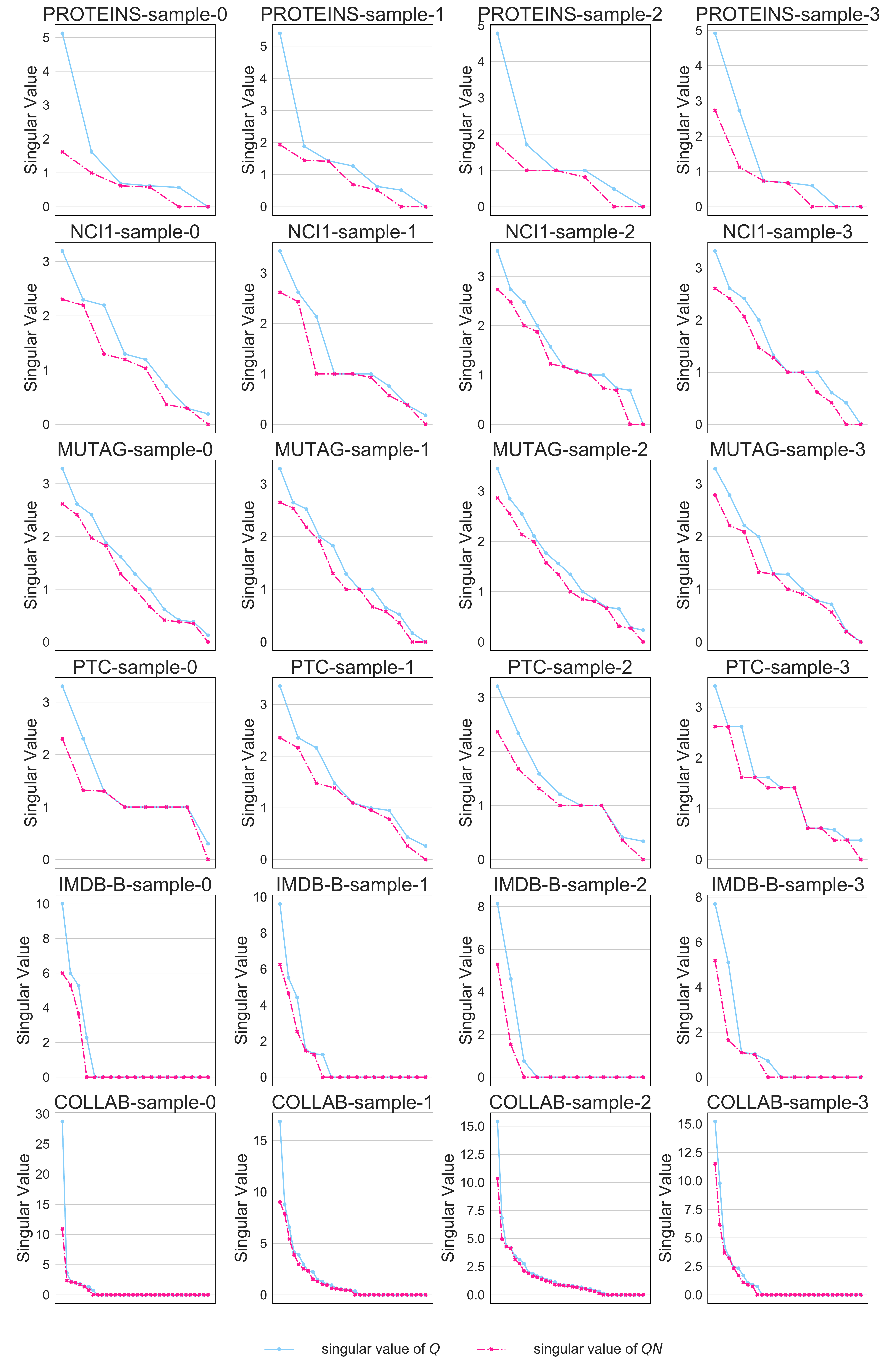}
    \caption{{\bf Singular value distribution of ${Q}$ and ${Q} N$}. Graph samples from PROTEINS, NCI1, MUTAG, PTC, IMDB-BINARY, COLLAB are presented. }
\label{fig:appendix-singular-value}
\end{figure*}

\begin{figure*}[ht]
    \centering
        \includegraphics[width=\textwidth]{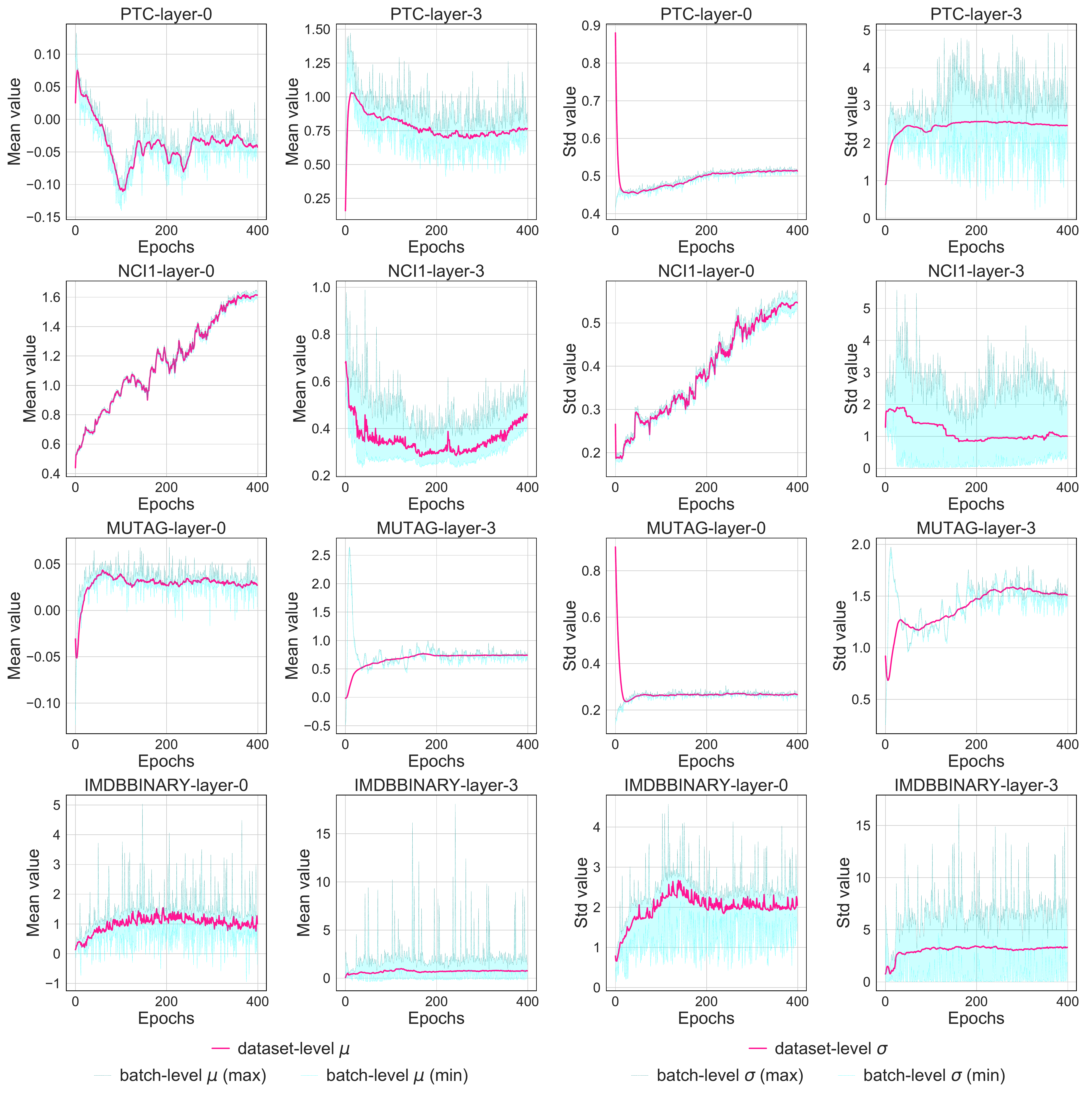}
    \caption{\textbf{Batch-level statistics are noisy for GNNs} (Examples from PTC, NCI1, MUTAG, IMDB-BINARY datasets). We plot the batch-level mean/standard deviation and dataset-level mean/standard deviation of the first (layer 0) and the last (layer 3) BatchNorm layers in different checkpoints. GIN with 5 layers is employed.}
\label{fig:appendix-noise-in-graphs-other-dts}
\end{figure*}

\begin{figure*}[ht]
    \centering
        \includegraphics[width=\textwidth]{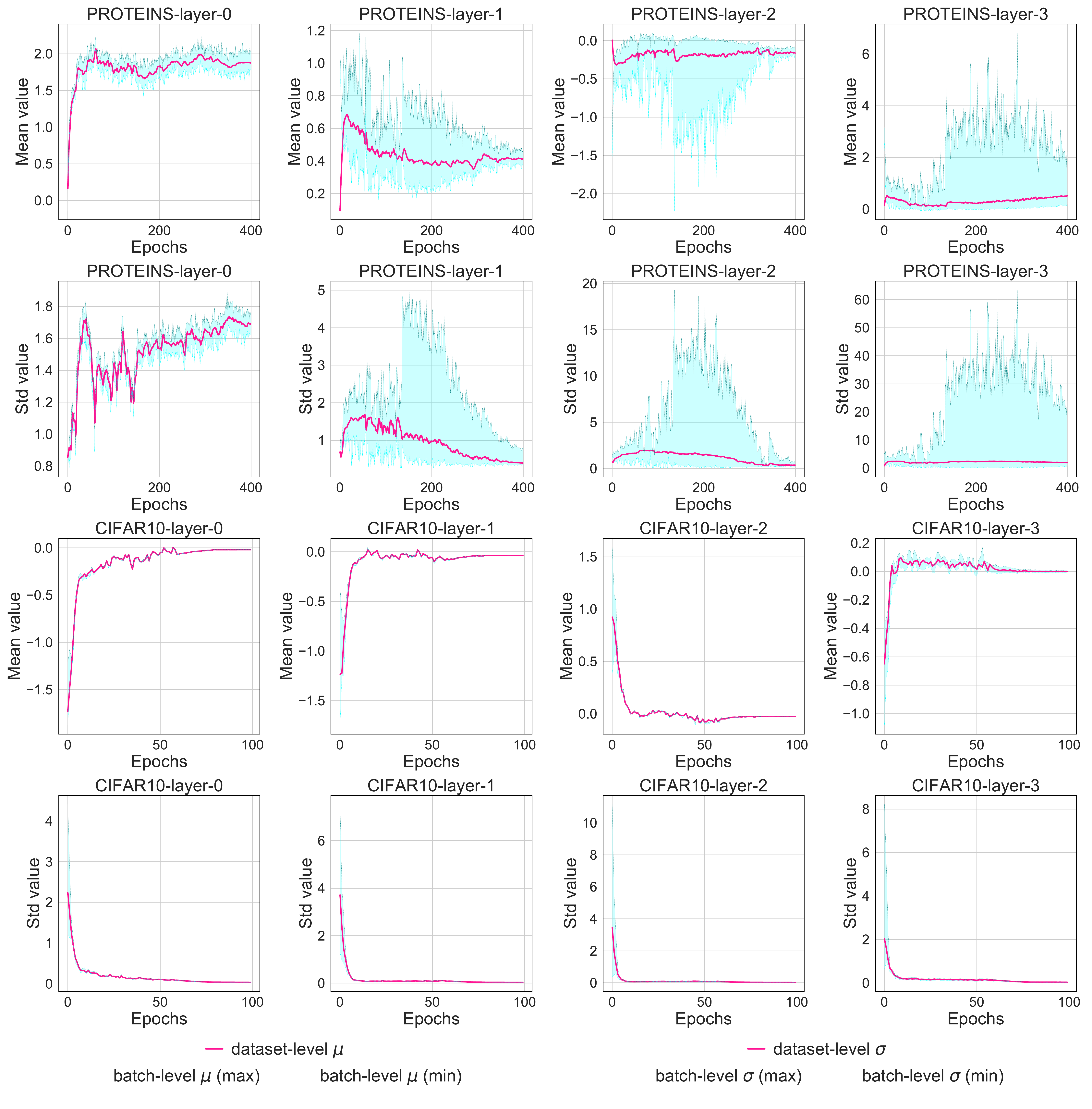}
    \caption{\textbf{Batch-level statistics are noisy for GNNs of different depth.} We plot the batch-level mean/standard deviation and dataset-level mean/standard deviation of different BatchNorm layers (from layer 0 to layer 3) in different checkpoints. We use a five-layer GIN on PROTEINS and ResNet18 on CIFAR10 for comparison.}
\label{fig:appendix-noise-in-graphs-layers}
\end{figure*}

\begin{figure*}[ht]
    \centering
        \includegraphics[width=\textwidth]{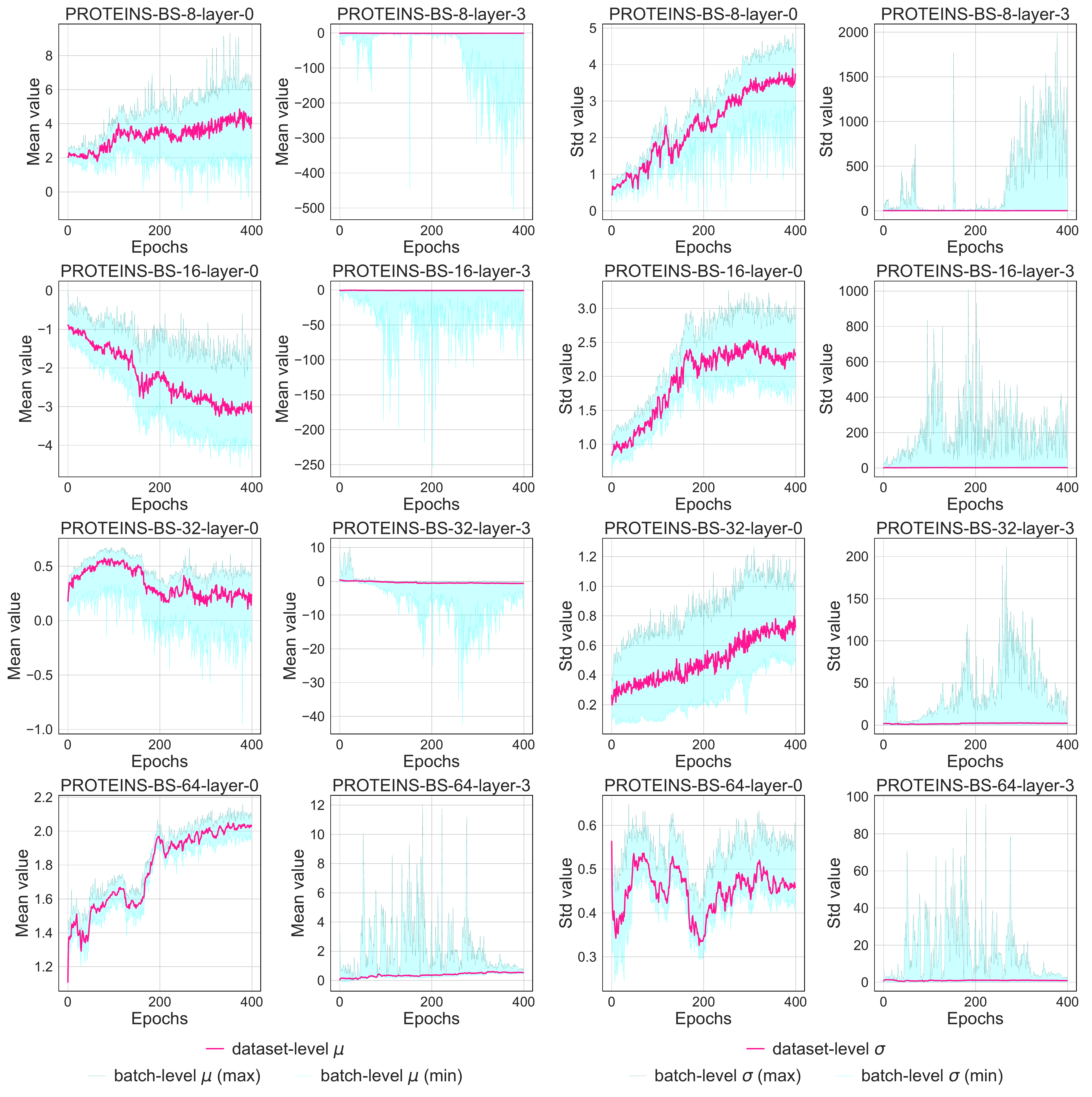}
    \caption{\textbf{Batch-level statistics are noisy for GNNs of different batch sizes.} We plot the batch-level mean/standard deviation and dataset-level mean/standard deviation of different BatchNorm layers (layer 0 and layer 3) in different checkpoints. Specifically, different batch sizes (8, 16, 32, 64) are chosed for comparison. GIN with 5 layers is employed.}
\label{fig:appendix-noise-in-graphs-batch-sizes}
\end{figure*}


\begin{figure*}[ht]
    \centering
        \includegraphics[width=\textwidth]{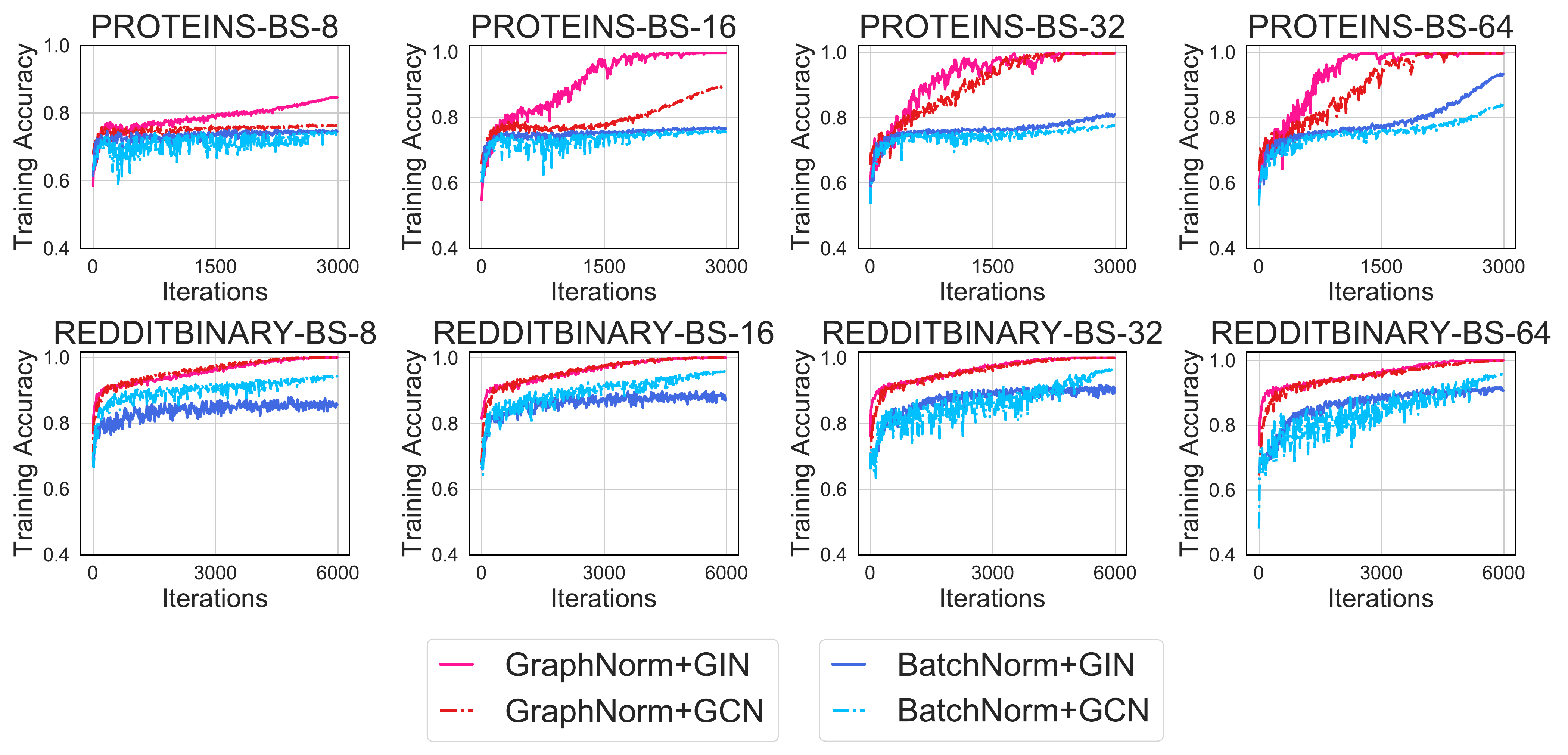}
    \caption{{\bf Training performance} of GIN/GCN with GraphNorm and BatchNorm with batch sizes of (8, 16, 32, 64) on PROTEINS and REDDITBINARY datasets.}
\label{fig:appendix-norm-bs-ablation-gin-gcn}
\end{figure*}

\begin{figure*}[ht]
    \centering
        \includegraphics[width=\textwidth]{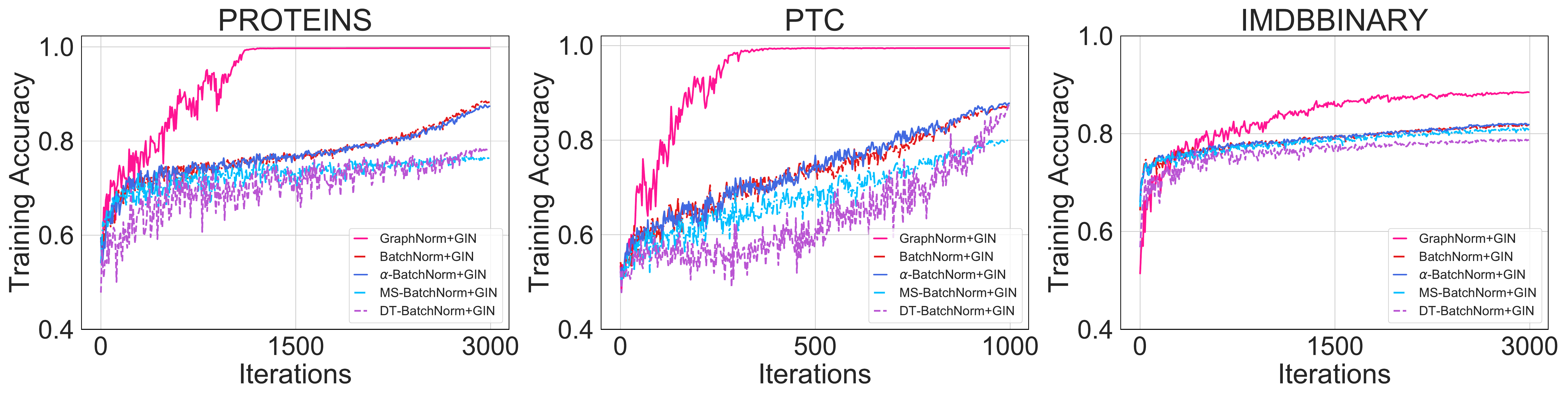}
    \caption{{\bf Training performance} of GIN with GraphNorm and variant BatchNorms ($\alpha$-BatchNorm, MS-BatchNorm and DT-BatchNorm) on PROTEINS, PTC and IMDB-BINARY datasets.}
\label{fig:appendix-norm-variant-ablation-gin-gcn}
\end{figure*}

\section{Other Related Works}
\label{appsec:related}
Due to space limitations, we add some more related works on normalization and graph neural networks here. \citet{zou2019layer} used normalization to stabilize the training process of GNNs. \citet{Zhao2020PairNorm:} introduced PAIRNORM to prevent node embeddings from over-smoothing on the node classification task. Our GraphNorm focuses on accelerating the training and has faster convergence speed on graph classification tasks. \citet{yang2020revisiting} interpreted the effect of mean subtraction on GCN as approximating the Fiedler vector. We analyze more general aggregation schemes, e.g., those in GIN, and understand the effect of the shift through the distribution of spectrum. Some concurrent and independent works \citep{li2020deepergcn,chen2020learning,zhou2020towards, zhou2020effective} also seek to incorporate normalization schemes in GNNs, which show the urgency of developing normalization schemes for GNNs. In this paper, we provide several insights on how to design a proper normalization for GNNs. Before the surge of deep learning, there are also many classic architectures of GNNs such as \citet{scarselli2008graph,bruna2013spectral,defferrard2016convolutional} that are not mentioned in the main body of the paper. We refer the readers to \citet{zhou2018graph,wu2020comprehensive,zhang2020deep} for surveys of graph representation learning.

\end{document}